\RequirePackage{etex} 
\documentclass[10pt]{article}

\usepackage[utf8]{inputenc}
\usepackage[framemethod=tikz]{mdframed}
\usepackage{amsmath,amsthm,amssymb,epsfig, psfrag, titlesec}
\usepackage[margin=1in]{geometry}
\usepackage{mathtools}
\usepackage{amsfonts}
\usepackage{textcomp}
\usepackage{pgfplots}
\usepackage{graphicx}
\usepackage{wrapfig}
\usepackage[]{youngtab}
\usepackage{amsthm}
\usepackage{tikz}
\usepackage{tikz-cd}
\usepackage{stmaryrd}
\usepackage{enumerate}
\usepackage{bbm}
\usepackage{float}
\usepackage{natbib}

\usepackage{eucal}
\usepackage{mathrsfs}
\pgfplotsset{width=10cm,compat=1.9}
\usepackage{hyperref}
\hypersetup{
    colorlinks=true,
    linkcolor=blue,
    citecolor=blue
}
\usepackage{cleveref}

\usepackage{algorithm}
\usepackage{verbatim}
\usepackage[noend]{algpseudocode}
\usepackage{xfrac}

 \usepackage{thmtools}
\usepackage{thm-restate}

\declaretheorem[name=Lemma]{lemma}
\declaretheorem[name=Definition]{definition}
\declaretheorem[name=Assumption]{assumption}
\declaretheorem[name=Example,style=definition]{example}
\declaretheorem[name=Remark,style=definition]{remark}
\declaretheorem[name=Proposition]{proposition}
\declaretheorem[name=Corollary]{corollary}
\declaretheorem[name=Claim]{claim}
\declaretheorem[name=Observation]{observation}

\declaretheorem[name=Question]{question}

\makeatletter
  \renewenvironment{proof}[1][Proof]%
  {%
   \par\noindent{\bfseries\upshape {#1.}\ }%
  }%
  {\qed\newline}
\makeatother

\newtheoremstyle{TH1}
  {\topsep}%
  {\topsep}%
  {\normalfont}%
  {}%
  {\bfseries}%
  {:}%
  {.5em}%
  {\thmname{#1}\thmnote{~(#3)}}%
\theoremstyle{TH1}

\DeclarePairedDelimiter{\crl}{\{}{\}}

\newmdtheoremenv[
skipabove=\baselineskip,
skipbelow=\baselineskip,
hidealllines=true,
innertopmargin=6pt,
linewidth=4pt,
linecolor=gray!40,
backgroundcolor=gray!8
]{exmpl}{{\sf Example}}

\definecolor{lightgray}{gray}{0.5}

\newcommand{\cB}{\mathcal{B}}

\newcommand{\cF}{\mathcal{F}}

\newcommand{\cH}{\mathcal{H}}

\newcommand{\cM}{\mathcal{M}}

\newcommand{\En}{\mathbb{E}}

\usepackage{enumerate}
\usepackage[shortlabels]{enumitem}

\newcommand{\eps}{\epsilon}

\newcommand{\bM}{\overline{M}}
\newcommand{\op}{\overline{p}}
\newcommand{\tM}{\tilde{M}}

\renewcommand{\bM}{\widehat{M}}
\newcommand{\amax}{a_{\text{max}}}
\newcommand{\minterm}{\min(\rho, \Delta + 1 - \gamma)}
\newcommand{\trho}{\widehat{\rho}}

\newcommand{\sss}[1]{{\scriptscriptstyle#1}}

\newcommand{\subs}[1]{_{{\scriptscriptstyle#1}}}

\newcommand{\pimhat}[1]{\pi_{\sss{\widehat{M}}}}
\newcommand{\Mstar}{M^{\star}}

\newcommand{\algcommentbig}[1]{\textcolor{blue!70!black}{\footnotesize{\texttt{\textbf{/*
          #1~*/}}}}}

\usepackage{autonum}

\title{Tight Bounds for $\gamma$-Regret via the Decision-Estimation Coefficient}
\author{Margalit Glasgow\\ {\small \texttt{mglasgow@stanford.edu}}\and Alexander Rakhlin\\ {\small \texttt{rakhlin@mit.edu}}}

\begin{document}
\maketitle

\begin{abstract}
In this work, we give a statistical characterization of the $\gamma$-regret for arbitrary structured bandit problems, the regret which arises when comparing against a benchmark that is $\gamma$ times the optimal solution. The $\gamma$-regret emerges in structured bandit problems over a function class $\cF$ where finding an exact optimum of $f \in \cF$ is intractable. Our characterization is given in terms of the $\gamma$-DEC, a statistical complexity parameter for the class $\cF$, which is a modification of the constrained Decision-Estimation Coefficient (DEC) of~\citet{foster2023tight} (and closely related to the original offset DEC of \citet{foster2021statistical}). Our lower bound shows that the $\gamma$-DEC is a fundamental limit for any model class $\cF$: for any algorithm, there exists some $f \in \cF$ for which the $\gamma$-regret of that algorithm scales (nearly) with the $\gamma$-DEC of $\cF$. We provide an upper bound showing that there exists an algorithm attaining a nearly matching $\gamma$-regret. Due to significant challenges in applying the prior results on the DEC to the $\gamma$-regret case, both our lower and upper bounds require novel techniques and a new algorithm. 
\end{abstract}

\section{Introduction}\label{sec:intro}
In this work, we study the problem of structured bandits. Formally, given a known class of functions $\cF$ on some domain $\Pi$, at each round $t = 1, \ldots T$, the algorithm queries some $\pi_t \in \Pi$ and achieves a random reward $r(\pi_t)$, where $\mathbb{E}[r(\pi_t)] = f^*(\pi_t)$ for some unknown ground truth function $f^* \in \cF\subseteq\{f:\Pi\to[0,1]\}$. The traditional goal in structured bandit problems is to minimize the regret
\begin{align}
    \mathsf{Reg}(T) := \sum_{t = 1}^T \max_{\pi \in \Pi}f^*(\pi) - f^*(\pi_t),
\end{align}
which compares the behaviour of the algorithm to the best action for $f^*$. In settings where maximizing $f^*(\pi)$ by better that a $\gamma$ ratio is intractable, we consider the notion of $\gamma$-regret, originally introduced in \citet{kakade2007playing, streeter2008online}\footnote{The definition of approximate regret in \citet{kakade2007playing} is slightly different since they consider minimization instead of maximization; our definition is equivalent to the $\gamma$-regret for \em maximization \em originally defined in \cite{streeter2008online}.}:
\begin{align}
    \mathsf{Reg}_{\gamma}(T) := \sum_{t = 1}^T \max_{\pi \in \Pi}\gamma f^*(\pi) - f^*(\pi_t).
\end{align}
This definition of regret compares the behavior of the algorithm to a benchmark which is a $\gamma$-fraction of the optimum, for some $\gamma\in(0,1]$. The need for studying the approximate notion of regret comes from non-convex or combinatorial optimization problems, where due to the exponentially large domain, finding exactly optimal solutions is computationally or statistically intractable in polynomial time. Yet, it is often feasible to find approximately optimal solutions, which attain a $\gamma$-fraction of the optimum. 
A few examples of combinatorial optimization problems where finding an exact optimum is intractable are various settings of submodular optimization, the traveling salesman problem, or clustering. Many such problems have been studied in an online setting with full-information or bandit feedback \citep{kakade2007playing, streeter2008online,blum2008regret,chen2016combinatorial,roughgarden2019minimizing,zhang2019online,fotakis2020efficient,harvey2020improved,dudik2020oracle,paria2021texttt,yang2021follow, foster2021submodular,azar2022alpha,perrault2022combinatorial, niazadeh2021online, nie2022explore}. 

Assuming for a minute that the rewards $r(\pi_t)$ follow a normal distribution with mean $f^*(\pi_t)$, each problem instance is fully described by the class $\cal F$ of possible truths $f^*$. The determination of minimax optimal \textit{sample complexity} then amounts to finding (nearly) matching upper and lower bounds $\psi, \Psi$ with
\begin{align}
\label{eq:minimax}
\psi(\cF, T) ~\leq~ \min_{\text{Alg}}\max_{f^*\in\cF} \En \left[ \mathsf{Reg}_{\gamma}(T) \right] ~\leq~ \Psi(\cF, T)
\end{align}
where the minimum is over all regret-minimization algorithms. To put this question in perspective, tight upper and lower bounds on minimax performance have been developed for a variety of learning problems, such as \textit{statistical and PAC learning} ---where $\psi,\Psi$ are formulated in terms of the VC dimension or entropy numbers of $\cF$---and  \textit{online learning}---where the analogous sequential complexities of $\cF$ can be shown to be necessary and sufficient. 

Recently, a line of work initiated by \citet{foster2021statistical} has yielded a general framework for determining the minimax regret of bandit problems (and more generally, online decision making problems such as reinforcement learning) in terms of a complexity measure called the \em Decision-Estimation Coefficient \em (DEC), described in Section~\ref{sec:dec_defn}. The DEC is shown in \citet{foster2021statistical, foster2022complexity, foster2023tight, FosRak22} to be a fundamental limit for the regret of online decision making problems, in the sense that it yields the exact optimal regret up to logarithmic factors in $T$ and in the size of the model class $\cF$. The DEC has proved to be a powerful framework for obtaining both upper and lower bounds for online decision making problems, for instance, unifying many existing results in reinforcement learning, and leading to new bounds and algorithms for contextual bandit problems. Our aim in this paper is to extend this framework to the case of $\gamma$-regret with $\gamma<1$. 

In general, our understanding of $\gamma$-regret is much poorer than our understanding of the traditional regret, and to our knowledge, there are no existing lower bounds on the $\gamma$-regret. Thus while many of the above works have given algorithms upper bounding the $\gamma$-regret of various structured bandit problems, we do not know whether they achieve the optimal $\gamma$-regret. For instance, for the setting of linear optimization, the works of \citet{kakade2007playing, garber2017efficient, hazan2018online} develop an online-to-offline reduction which achieves $T^{2/3}$ $\gamma$-regret (for minimization) in bandit settings using an offline approximate optimization oracle. However, it is possible that a $\sqrt{T}$ $\gamma$-regret is achievable, because in the analogous $1$-regret linear bandit setting, where an exact offline optimization oracle is available, the optimal regret is known to scale with $\sqrt{T}$~\citep{ito2019oracle}. One reason proving lower bounds in the $\gamma$-regret settings is challenging is because at each round $t$, it is possible to achieve \em negative \em $\gamma$-regret if $f^*(\pi_t) \geq \gamma \max_{\pi \in \Pi} f^*(\pi)$. Thus to prove a lower bound of, say, $\Delta T$ on the $\gamma$-regret, it no longer suffices to show that for, say, $T/2$ rounds, we have $\gamma$-regret greater than $2\Delta$, which is a standard approach in bandit lower bounds (see eg. \cite{lattimore2020bandit}).

While the DEC appears to be a promising candidate for understanding the $\gamma$-regret, the existing results in \citet{foster2021statistical} and \citet{foster2023tight} fall short of characterizing the $\gamma$-regret. Briefly, the reason for this is that the lower bounds in both works, and the upper bound in \citet{foster2023tight}, is only tight up to constant factors at best.
Thus, if we desire to show that the $\gamma$-regret is on the order of $\Delta T$ -- which is similar to the traditional regret being $\left(1 - \gamma + \Theta(\Delta)\right)T$ -- any upper or lower bound which loses constant factors will only be able to show a traditional regret of $\Theta((1 - \gamma + \Delta)T)$, which does not translate to any meaningful bound on the $\gamma$-regret if $\Delta = o(1)$.

Our main result is showing that a variant of the DEC, which we term the $\gamma$-DEC, does in fact characterize the $\gamma$-regret, in the sense of providing nearly matching upper and lower bounds in \eqref{eq:minimax}. We formally define the $\gamma$-DEC, parameterized by $\epsilon$ and denoted  $\mathsf{dec}^{\gamma}_{\eps}$, in Definition~\ref{def:decc}. In the case that $\gamma = 1$, the $\gamma$-DEC essentially generalizes the constrained DEC of \citet{foster2023tight}. 

Specifically, our first result, Theorem~\ref{prop:decc_rho}, lower bounds the regret by $T \mathsf{dec}^{\gamma}_{\eps}$ (for $\eps \approx \sqrt{1/T}$) with exact multiplicative constant $1$ up to a small \em additive \em term, that depends on a certain localization parameter of the model class $\cF$. Even in the absence of localization, for the traditional exact regret setting, our result yields an improvement upon the lower bound of $\mathsf{dec}^1_{\eps}/ \Theta(\log(T))$ of \citet{foster2023tight} by a logarithmic factor. More importantly, the sharp nature of the $\mathsf{dec}^{\gamma}_{\eps}$ lower bound allows us to establish the corresponding lower bounds on $\gamma$-regret. The proof of Theorem~\ref{prop:decc_rho} involves a new technique of lower bounding the regret by considering the behavior of the algorithm up to various stopping times $\tau$. 

Our second result, Theorem~\ref{thm:ub}, provides an algorithm which obtains a regret of at most $T \mathsf{dec}^{\gamma}_{\eps}$ (for $\eps \approx \sqrt{1/T}$), matching our lower bound. This algorithm is based on the estimation-to-decision principle from prior work on the DEC, which employs a reduction from algorithms for interactive decision making (such as bandits) to online estimation algorithms, which provide an estimator of $f^*$ based on past observations. Thus Theorem~\ref{thm:ub} assumes access to such an online estimation oracle, described in Assumption~\ref{assm:oracle}. While our algorithm is similar to those in \citet{foster2021statistical, foster2023tight}, it differs by carefully tuning a certain parameter which governs the amount of exploration at each round, allowing us achieve a meaningful result in the $\gamma$-regret setting.

We emphasize that the matching upper and lower bounds \eqref{eq:minimax} are obtained for any set of models (satisfying the mild assumptions stated in Section~\ref{sec:main_theorems}) and thus constitute a complete solution to determining the minimax sample complexity for a function class $\cF$, at least from the statistical point of view. In Section~\ref{sec:comp}, we mention some of the computational implications of our results, and discuss some exciting directions at the intersection of statistics and computation. Finally in Section~\ref{sec:ex}, we show how to bound the $\gamma$-DEC in several examples. As one example, we consider a bandit problem over the action space $\mathbb{R}^d$ for which finding an exact optimum in $T \leq \exp(\Theta(d))$ steps is impossible, but finding a $\gamma$-approximate-maximum is possible. We show how to derive a tight bound on the $\gamma$-DEC, and show that Theorem~\ref{prop:decc_rho} can yield a tight lower bound for the $\gamma$-regret of this bandit problem. We believe that bounding the $\gamma$-DEC is a promising framework for understanding the $\gamma$-regret in more natural approximate optimization settings in the references above, and we leave to future work the question of finding the $\gamma$-DEC in such applications.

\subsection{Organization}
In Section~\ref{sec:main_theorems}, we state preliminaries and notation, and our main theorems. In Section~\ref{sec:dec} we give an overview of the proof of our lower bound on the regret in terms of the $\gamma$-DEC. In Section~\ref{sec:ub} we give an overview of the proof of our upper bounds on the regret in terms of the $\gamma$-DEC. In Section~\ref{sec:rw}, we compare our definitions and results to related work on the DEC. In Section~\ref{sec:ex}, we give some examples which apply our results. In Section~\ref{sec:dec_lb}, we prove our lower bound. In Section~\ref{sec:ubproof} we prove our upper bounds.

\section{Preliminaries and Theorem Statements}\label{sec:main_theorems}
\subsection{Structured Bandits and the Decision-Estimation Coefficient.}\label{sec:dec_defn}
We define our notation to keep with the prior line of work on the DEC. We study structured bandit problems with action space $\Pi$ and a set of models $\cM$. Each model $M\in\cM$ is a probability kernel $M:\Pi\to\Delta(\mathbb{R})$ where $\Delta(\mathbb{R})$ is a set of distributions on $\mathbb{R}$. The set of models induces a class of mean reward functions  $\cF= \{f_M\}_{M \in \cM}$ where $f_M(\pi) = \mathbb{E}^M[r|\pi]=\mathbb{E}\subs{r\sim M(\pi)}[r]$ and $\mathbb{E}^M$ denotes the expectation under the model $M$. We also use $\mathbb{P}^M$ to denote the probability operator under the model $M$. For a model $M \in \cM$, we define $\pi_M := \arg \max_{\pi \in \Pi} f_M(\pi)$. We let $M^*\in\cM$ denote the true model, and the corresponding mean reward function by $f^*=f_{M^*}$. Throughout, we assume $f_M(\pi)\in[0,1]$ for all $\pi\in\Pi, M\in\mathcal{M}$. We use $\mathsf{co}(\cM)$ to denote the convex hull of a model class $\cM$.

At each round $t = 1, \ldots, T$, the algorithm chooses an action $\pi_t \in \Pi$. The algorithm then receives a reward $r_t$, where $\mathbb{E}[r_t | \pi_t] = f^*(\pi_t)$, and the variance of $r_t$ is at most $1$. Formally, we define the history $\cH_t := \{(\pi_i, r_i)\}_{i = 1}^t$, and we let $\{p_t\}_{t=1}^T$ be any randomized algorithm which maps histories to distributions over $\Pi$, that is, $\pi_t \sim p_t(\cH_{t - 1})$. 

We now define the \em $\gamma$-Decision-Estimation Coefficient \em (DEC), which generalizes\footnote{See discussion in Section~\ref{rem:hellinger}; since we use a squared error to $f_{\bM}$ constraint instead of a squared Hellinger distance constraint, the definition stated here generalizes theirs only for the bandit setting with Gaussian noise.}
the constrained DEC, recently introduced in \citet{foster2023tight}. 
\begin{definition}[$\gamma$-DEC]\label{def:decc}
For a model class $\cM$ and a reference model $\bM$, for any $\gamma \in (0, 1]$ and $\eps \in [0, 1]$, define 
\begin{equation} \label{eq:def_dec}
    \mathsf{dec}^{\gamma}_{\eps}(\cM, \bM) := \min_{p \in \Delta(\Pi)} \max_{M \in \cM} \crl*{ \mathbb{E}_{\pi \sim p}\left[\gamma f_M(\pi_M) - f_M(\pi)\right] \mid M \in \cB_{p, \eps}(\bM)},
\end{equation}
where $\cB_{p, \eps}(\bM) := \{M \in \cM \cup \bM: \mathbb{E}_{\pi \sim p}(f_M(\pi) - f_{\bM}(\pi))^2 \leq \eps^2\}$ is set of models with mean rewards within an $L_2(p)$ ball of radius $\eps$ around $\bM$.
Define 
\begin{equation}
    \label{eq:def_max_dec}
    \mathsf{dec}^{\gamma}_{\eps}(\cM) := \sup_{\bM} \mathsf{dec}^{\gamma}_{\eps}(\cM \cup \{\bM\}, \bM),
\end{equation}
where the supremum is over any probability kernel $\bM: \Pi \rightarrow \Delta(\mathbb{R})$, not necessarily in $\cM$.
\end{definition}

The $\gamma$-DEC quantifies the best possible instantaneous $\gamma$-regret attained by a distribution $p$ over $\Pi$ under the worst model $M$ that is close in the sense of expected squared error to $f_{\bM}$. This quantity captures the exploration-exploitation trade-off of $\cM$: if there exists some distribution $p$ under which for each $M \in \cM \cup \{\bM\}$ either the information gain ($\mathbb{E}[(f_M(\pi) - f_{\bM}(\pi))^2]$) is large, or the $\gamma$-regret ($\mathbb{E}[\gamma f_M(\pi_M) - f_{M}(\pi)]$) is small, then $\mathsf{dec}^\gamma_{\eps}$ will be small. 
The work of \citet{foster2023tight} showed that the quantity $\mathsf{dec}^1_{\eps}$ for $\eps = \sqrt{1/T}$ characterizes the $1$-regret of interactive decision making, in the sense that $\frac{T\mathsf{dec}^1_{\eps}}{\Theta(\log(T))}\leq \mathbb{E}[\mathsf{Reg}(T)] \leq \Theta\left(T\mathsf{dec}^1_{\eps}\log(|\cM|)\right)$, under mild conditions.  In this work, we extend these results on traditional regret to the setting of $\gamma$-regret. We refer additionally to \citet{foster2021statistical, FosRak22} for a detailed discussion on the relationship between DEC and notions such as Posterior Sampling and the Information Ratio, Feel-Good Thompson Sampling, the modulus of continuity in statistical estimation, and further connections to the literature.

\subsection{Lower Bound}

Our first result, Theorem~\ref{prop:decc_rho}, yields a constant probability lower bound on the regret in terms of the $\gamma$-DEC. In order for this theorem to be meaningful in the $\gamma$-regret setting, we will leverage the following localization property, which states that all models in $\cM$ have similar maxima.

\begin{definition}\label{def:rho}
A model class $\cM$ is $\rho$\em -localized \em if $\max_{M \in \cM}f_M(\pi_M) - \min_{M \in \cM} f_M(\pi_M) \leq \rho$.
\end{definition}
We note that the $\gamma$-regret is monotonically decreasing with respect to taking smaller subsets of a model class. Thus, for the purposes of achieving a better lower bound for a model class with a poor localization parameter, one may wish to consider the $\gamma$-DEC of a more localized subset of the model class.

Without loss of generality, we will assume throughout that in our model class $\cM$, we have $\sup_{M \in \cM} f_M(\pi_M) = 1$. 

\begin{restatable}{theorem}{thmdec}\label{prop:decc_rho}
Let $C \geq 2$ be any integer. Let $\frac{4}{T} \leq \eps \leq \frac{1}{3C\sqrt{T}}$. 
Let $\Delta := \mathsf{dec}^{\gamma}_{\eps}(\cM)$. Suppose $\cM$ is $\rho$-localized for some $\rho \geq 6C^2\eps$ and that $\Delta \geq 6C^2\eps$. Then for any algorithm, for some $M \in \cM$, 
$$\mathbb{P}^M\left[\frac{1}{T}\mathsf{Reg}_{\gamma}(T)\geq \Delta - \frac{3\min(\rho, \Delta + 1 - \gamma)}{C}\right] \geq \frac{1}{3C}.$$
Here the probability is over the randomness of the algorithm and variance-1 Gaussian noise in the rewards.
\end{restatable}

\begin{remark}
If we invoke Theorem~\ref{prop:decc_rho} with $\rho = 1$ and $\gamma = 1$ (that is, without any localization assumption on $\cM$, and with the constrained DEC as defined in \cite{foster2023tight}), and choose $C = 6$, we achieve that for some $M \in \cM$,
$\mathbb{P}\left[\mathsf{Reg}(T)\geq \frac{\Delta}{2}\right] \geq \frac{1}{18}$, implying that $\mathbb{E}\left[\mathsf{Reg}(T)\right] \geq \frac{\Delta}{36}$. This improves upon the lower bound in \citet[Theorem 2.2]{foster2023tight}, which yields $\mathbb{E}\left[\mathsf{Reg}(T)\right] \geq \frac{\Delta}{\Theta(\log(T))}$, by removing the log factors. Remarkably, for model classes $\cM$ with the same value $f_M(\pi_M)$ for all $M \in \cM$, the lower bound has a sharp factor of $1$ in front of the DEC, up to additive terms that scale with $\eps$.
\end{remark}

\subsection{Upper Bounds}

Having established a sharp lower bound for \eqref{eq:minimax} in terms of $\gamma$-DEC, we now turn to the question of upper bounds. We prove two upper bounds on the regret in terms of the $\gamma$-DEC. Our first upper bound, Theorem~\ref{thm:simpleub}, can be viewed as a warm up, as it only yields a rate matching the lower bound in the standard regret setting when $\gamma = 1$. For the case of $\gamma = 1$, this Theorem~\ref{thm:simpleub} matches the rate given in \citet{foster2023tight}. We include it both as a warm-up and because the algorithm yielding this regret bound is considerably simpler than the algorithm given in \citet{foster2023tight}. Our second upper bound, Theorem~\ref{thm:ub}, makes an additional continuity assumption, but is able to achieve a rate that matches our lower bound for any approximation ratio $\gamma$.

Similarly to the algorithms in \citet{foster2020beyond,foster2022complexity, foster2023tight}, our upper bounds are based on a reduction from decision making to estimation of $M^*$, and thus rely on an online estimation oracle which outputs a model $\bM$ based on the past observations of the algorithm. Our upper bounds will depend both on the $\gamma$-DEC, and on the estimation error of this oracle. The following assumption captures the guarantee of our estimation oracle. 

\begin{assumption}[Online Estimation Oracle]\label{assm:oracle}
There exists an online estimation oracle $\mathsf{AlgEst}$ for $\cM$, which, when given a history $\cH_{t-1} = \{\pi_s, r_s\}_{s = 1}^{t - 1}$ with $\pi_s \sim p_s$, and $r_s \sim M^*(\pi_s)$, returns a model $\bM_t \in \mathsf{co}(\cM_{t-1})$, such that if $\Mstar \in \cM$, with probability $1 - \delta$,
\begin{align}   
\sum_{t = 1}^T \mathbb{E}_{p_t}[(f_{M^*}(\pi_t)- f_{\bM_t}(\pi_t))^2] \leq \mathsf{Est}(T, \delta),
\end{align}
where $\cM_t := \{M \in \cM: \sum_{s = 1}^{t} \mathbb{E}_{p_s}[(f_{M^*}(\pi_s)- f_{\bM_s}(\pi_s))^2] \leq \mathsf{Est}(T, \delta)\}$.
\end{assumption}
\begin{remark}
Assumption~\ref{assm:oracle} is also made for the upper bound in \citet{foster2021statistical} (see their Assumption D.1). As an example, for finite classes $\cM$, a minor modification to the exponential weights algorithm yields an online estimation oracle satisfying this assumption with $\mathsf{Est}(T, \delta) \leq O\left(\log(|\cM|/\delta)\right)$; see \citet{foster2021statistical}, Lemma A.18.
\end{remark}

We additionally assume the following mild growth condition on the DEC. This condition holds in regimes where the $\gamma$-regret grows at a rate of $\sqrt{T}$ or greater. 

\begin{assumption}[Regularity]\label{assm:growth}
There exists constants $C > 1$ and $0 \leq c \leq 1$ such that $\mathsf{dec}^{\gamma}_{\eps} \leq \Delta_{\eps} := C\eps^c$ for all $\eps > 0$.

\end{assumption}
We state our first upper bound, which holds under these assumptions. Similarly to the lower bounds, we assume without loss of generality throughout that $\sup_{M \in \cM} f_M(\pi_M) \leq 1$.

\begin{restatable}{theorem}{simpub}\label{thm:simpleub}
Suppose Assumption~\ref{assm:oracle} hold and Assumption~\ref{assm:growth} holds for some function $\Delta_{\eps} \geq \mathsf{dec}^{\gamma}_{\eps}$. With probability $1 - \delta$, Algorithm~\ref{alg:1} attains regret $$\mathsf{Reg}_{\gamma}(T) \leq (1 - \gamma)T + 2T\Delta_{\eps} + 9\sqrt{T\mathsf{Est}(T, \delta)},$$
where $\eps = \sqrt{24\mathsf{Est}(T, \delta)/T}$.
\end{restatable}
Algorithm~\ref{alg:1} is conceptually simple: at each round $t = 1, \ldots, T$, we feed the past observations to the online estimation oracle to obtain some model $\bM_t$. We then choose an action $\pi_t$ from the distribution $p_t$ which minimizes $\mathsf{dec}_{\eps_t}^{\gamma}$ for the reference model $\bM_t$, and the parameter $\eps_t \propto 1/\sqrt{t}$.

\begin{algorithm}[t]
\caption{}\label{alg:1}
\textbf{input:}  Parameter $\gamma \in [0, 1]$, online estimation oracle $\mathsf{AlgEst}$ with known estimation guarantee $\mathsf{Est}(T, \delta)$.
\begin{algorithmic}[0]
    \State 
    Initialize an online estimation oracle for the model class $\cM$.
    \For{$t=1, \ldots, T$}
    \State Let $\bM_t = \mathsf{AlgEst}(\mathcal{H}_{t-1})$, such that $\bM_t \in \mathsf{co}(\mathcal{M}_{t - 1})$, as defined in Assumption~\ref{assm:oracle}.
   \State Let $p_t$ be the distribution that minimizes $\mathsf{dec}_{\eps_t}^\gamma(\cM \cup \bM_t, \bM_t)$, where $\eps_t = \sqrt{\frac{12\mathsf{Est}(T, \delta)}{t}}$. 
   \State Play $\pi_t \sim p_t$.
    \EndFor
\end{algorithmic}
\end{algorithm}

In the case that $\gamma = 1$, Theorem~\ref{thm:simpleub} yields an upper bound on the regret which matches our lower bound in Theorem~\ref{prop:decc_rho} up to constant factors, and the dependence on the estimation complexity $\mathsf{Est}(T, \delta)$. However, in the case that $\gamma < 1$, Theorem~\ref{thm:simpleub} gives a bound on the $\gamma$-regret that grows linearly in $T$. 

We next show in Theorem~\ref{thm:ub} that it is possible to remove the $(1 - \gamma)T$ term in the regret under an additional continuity assumption, stated below.

\begin{assumption}[Continuity]\label{assm:cont}
There exists a function $p(\eps, \bM)$ which is \em continuous \em in $\epsilon$ and returns a distribution over $\Pi$ which upper bounds the $\gamma$-DEC by at most $\Delta_{\eps}$, that is, with $p = p(\eps, \bM)$:
\begin{align}
    \sup_{M \in \cB_{p, \eps}(\bM)} f_M(\pi_M) - \mathbb{E}_{\pi \sim p}f_M(\pi) \leq \Delta_{\eps}.
\end{align}
\end{assumption}

The algorithm we use, Algorithm~\ref{alg:best} (stated in Section~\ref{sec:ubproof}), is similar to Algorithm~\ref{alg:1}, but chooses the schedule $\eps_t$ by solving a certain optimization problem at each step. Using  Algorithm~\ref{alg:best} attains the following result.

\begin{restatable}{theorem}{bestub}\label{thm:ub}
Suppose Assumption~\ref{assm:oracle} holds with estimation guarantee $\mathsf{Est}(T, \delta)$, and Assumptions~\ref{assm:growth} and \ref{assm:cont} hold for some function $\Delta_{\eps}$. With probability $1 - \delta$, Algorithm~\ref{alg:best} attains the regret
$$\mathsf{Reg}_{\gamma}(T) \leq  500\log(T)\Delta_{\sqrt{\frac{\mathsf{Est}(T, \delta)}{T}}} + \sqrt{T\mathsf{Est}(T, \delta)}.$$
\end{restatable}

\begin{remark}
We believe Assumption~\ref{assm:cont} is a natural assumption for settings in which the $\gamma$-DEC can be minimized efficiently. For instance, in Section~\ref{sec:ex_mab}, we show that for the case of the multi-armed bandit, Assumption~\ref{assm:cont} holds with $\Delta_{\eps} \approx \mathsf{dec}_{\eps}^{\gamma}$; in Section~\ref{sec:apxoracles}, we give an example for upper bounding the $\gamma$-DEC via a continuous distribution in linear optimization settings with an approximate optimization oracle.
\end{remark}

\subsection{Computational Perspective}\label{sec:comp}
While our results thus far have been statistical --- yielding (nearly) matching upper and lower bounds on the minimax optimal $\gamma$-regret --- it is natural to consider whether there is a computational analog of our results. In particular, we can ask the following two questions.
\begin{question}\label{q1}
In settings where it is computationally tractable to obtain a $\gamma$-regret of $\Delta T$, is it always possible to efficiently find a distribution $p$ that upper bounds the $\gamma$-DEC by nearly $\Delta$?
\end{question}
\begin{question}\label{q2}
Conversely, if it is possible to efficiently construct a distribution $p$ which upper bounds the $\gamma$-DEC by $\Delta$, can we provide an efficient algorithm which bounds the $\gamma$-Regret by nearly $\Delta T$? 
\end{question}
These questions are particularly relevant in the $\gamma$-regret setting, where often times the barrier to obtaining sublinear traditional regret is \em computational, \em not statistical. If the answers were affirmative, it would suggest that upper bounding the $\gamma$-DEC efficiently is (nearly) equivalent to efficiently minimizing $\gamma$-regret.

While the current tools in this work are not yet sufficient to answer these questions in a completely satisfying way, we state the following computational corollaries to our results, which begin to give a picture of computational significance of the DEC. To state these results, we define a computational analog of the $\gamma$-DEC.
\begin{definition}
We say an algorithm $A$ upper bounds $\mathsf{dec}_\eps^{\gamma}(\cM)$ by $\Delta$ in time $R$, if with query access to $f_{\bM}$ for any model $\bM \in \mathsf{co}(\cM)$, in time $R$, with probability $1 - \delta$, $A$ can provide a random sample from a distribution $p$ for which
\begin{align}
\max_{M \in \cM \cup \{\bM\}} \crl*{ \mathbb{E}_{\pi \sim p}\left[\gamma f_M(\pi_M) - f_M(\pi)\right] \mid M \in \cB_{p, \eps}(\bM)} \leq \Delta.
\end{align}

We abuse notation and define $\mathsf{dec}_\eps^{\gamma}(\cM, (R, \delta))$ to be the best upper bound on $\mathsf{dec}_\eps^{\gamma}$ by any algorithm $A$ that runs in time $R$ and succeeds with probability at least $1 - \delta$.
\end{definition}

\begin{definition}
    We say a bandit algorithm $\{p_t\}_{t = 1}^T$, where $p_t: (\Pi \times \mathbb{R})^{t - 1} \rightarrow \Delta(\Pi)$ maps histories $\cH_{t-1}$ to distributions over $\Pi$, runs in time $R$ if for all $t$, on any history $\cH_{T} \in (\Pi \times \mathbb{R})^{T}$, the total runtime of $p_t(\cH_{t - 1})$ for all $t \leq T$ is at most $R$. Here $\cH_{t}$ is defined to be the first $t$ pairs in $\cH_T$.
\end{definition}

We prove the following corollary of our lower bound. It requires the following condition.
\begin{assumption}[Efficient Optimization]\label{assm:eo}
With query access to any $f_{\bM}$ for $\bM \in \mathsf{co}(\cM)$, there is an algorithm which runs in time $R$ which outputs some $\hat{\pi}$ with $f_{\bM}(\hat{\pi}) \geq f_{\bM}(\pi_{\bM}) - \Delta/10$.
\end{assumption}
\begin{restatable}[Computational Corollary to Theorem~\ref{prop:decc_rho}]{corollary}{compp}\label{corr:comp}
Let $C \geq 2$ be any integer, and let $\eps = \frac{1}{3C\sqrt{T}}$ and $\delta > 0$. 
Suppose $\mathsf{dec}^{\gamma}_{\eps}(\cM, (2R, \delta)) \geq \frac{11}{10}\Delta$. Suppose $\cM$ is $\rho$-localized for some $\rho \geq 6C^2\eps$. Assume also that $\Delta \geq 6C^2\eps$. Then if Assumption~\ref{assm:eo} holds, for any algorithm $\{p_t\}$ which runs in time $\frac{R}{50C^2\log(C/\delta)}$, for some $M \in \cM$, 
$$\mathbb{P}^M\left[\frac{1}{T}\mathsf{Reg}_{\gamma}(T)\geq \Delta - \frac{3\min(\rho, \Delta + 1 - \gamma)}{C}\right] \geq \frac{1}{6C}.$$
Here the probabilities are over the randomness of the algorithm and variance-1 Gaussian noise in the rewards.
\end{restatable}
Corollary~\ref{corr:comp} gives a lower bound on the regret of any efficient algorithm in terms of the computation lower bound on the $\gamma$-DEC. The contrapositive of this corollary provides a conditional answer to Question~\ref{q1}: if there is an efficient $\gamma$-regret minimizing algorithm, then conditional on it being efficient to find a near-optimal maximizer of $f_{\bM}$ for any $\bM \in \mathsf{co}(\cM)$ (Assumption~\ref{assm:eo}), it is efficient to find a distribution $p$ upper bounding the $\gamma$-DEC with query access to $f_{\bM}$. This answer is somewhat unsatisfying in the $\gamma \neq 1$ case, because it would be impractical to assume Assumption~\ref{assm:eo}. However, in the case that $\gamma = 1$, Assumption~\ref{assm:eo} is indeed \em necessary \em for efficiently upper bounding the DEC by $\Delta/10$, because we always have $\bM \in \cB_{p, \eps}(\bM)$. Thus, under Assumption~\ref{assm:eo}, the answer to Question~\ref{q1} is yes; otherwise, it is no. Corollary~\ref{corr:comp} is proved in Section~\ref{sec:dec_lb}.

We now turn to the converse, Question~\ref{q2}. We can prove the following corollary to our upper bound in Theorem~\ref{thm:simpleub}. We focus on our attention on Theorem~\ref{thm:simpleub}, which only yields a matching bound when $\gamma = 1$, because the result is simpler, and as previously discussed for the lower bound, our understanding of the computational landscape is limited for $\gamma \neq 1$. 
\begin{corollary}[Computational Corollary to Theorem~\ref{thm:simpleub}]\label{corr:compub}
Suppose Assumption~\ref{assm:oracle} holds and Assumption~\ref{assm:growth} holds for some function $\Delta_{\eps} \geq \mathsf{dec}^{\gamma}_{\eps}(\cM, (R, \delta))$. With probability $1 - \delta(T+1)$, Algorithm~\ref{alg:1} runs in time $RT$ and makes $T$ oracle calls to the online estimation oracle (Assumption~\ref{assm:oracle}), and attains regret
\begin{align}
    \mathsf{Reg}_{\gamma}(T) \leq (1 - \gamma)T + 2T\Delta_{\eps} + 9\sqrt{T\mathsf{Est}(T, \delta)},
\end{align}
where $\eps = \sqrt{24\mathsf{Est}(T, \delta)/T}$.
\end{corollary}
Corollary~\ref{corr:compub} provides a conditional answer to Question~\ref{q2} in the case that $\gamma = 1$. It answers affirmatively whenever there is an \em efficient \em online optimization oracle satisfying Assumption~\ref{assm:oracle}. (More precisely, the online algorithm implementing Assumption~\ref{assm:oracle} should efficiently output query access to $f_{\bM_t}$). Corollary~\ref{corr:compub} is immediate from examining Algorithm~\ref{alg:1} and Theorem~\ref{thm:simpleub}, so we omit its proof.

We leave it as an exciting open direction to provide satisfying answers to Questions~\ref{q1} and \ref{q2} for the $\gamma \neq 1$ case. 

\section{Technical Overview of Lower Bound (Theorem~\ref{prop:decc_rho})}\label{sec:dec}
We present here a sketch of the proof of Theorem~\ref{prop:decc_rho}. For simplicity, in this section we consider the case when $\gamma = 1$, since the main idea of our technique is clearer in this setting. 

Fix an algorithm $\{p_t\}$, and for any model $M$, define $g_M(\pi) := f_M(\pi_M) - f_M(\pi)$. Let $\bM$ be the model attaining the constrained DEC in \eqref{eq:def_max_dec}, that is, for any distribution $p$ on $\Pi$, for some model $M \in \cM \cup \{\bM\} $, 
\begin{align}
    \label{eq:lower_bd_gm}
    \mathbb{E}_{\pi \sim p}[g_M(\pi)] \geq \Delta,
\end{align}
and $\mathbb{E}_{\pi \sim p}[(f_M(\pi) - f_{\bM}(\pi))^2]\leq \eps^2.$

A typical approach (see eg. \citet{foster2021statistical, foster2023tight}) is to invoke the lower bound \eqref{eq:lower_bd_gm} on the distribution $p=p_{\bM}$, defined to be the average distribution of actions played over all $T$ steps by the algorithm $\{p_t\}$ under the ground truth $\bM$. However, this approach results in the loss of constant factors, and only suffices if $\bM$ is contained in the model class $\cM$. To mitigate these issues, we will find a different distribution to invoke the DEC lower bound on, one which we know will result in $M$ not being $\bM$ (and thus necessarily $M \in \cM$). To construct such a distribution, we consider the algorithm $\{\op_t\}$, which plays algorithm $\{p_t\}$ until the \em stopping time \em  $\tau$, and then plays $\pi_{\bM}$ for the rest of the $T - \tau$ steps. Assuming the stopping time $\tau$ occurs before $\bM$ has accrued $\Delta T$ regret, it will be impossible for $\bM$ to attain the DEC. Thus our goal will be to show that when we invoke the DEC on $\op_{\bM}$ (the distribution of $\op_t$ averaged over all $T$ rounds), the model $M$ maximizing the DEC attains, with constant probability, a regret close to $\Delta T$.

For any model $\tM$, recall that $\mathbb{P}^{\tM}$ denotes the probability operator under the ground truth $\tM$, and algorithm $\{p_t\}$. Similarly, define $\overline{\mathbb{P}}^{\tM}$ to be the probability operator under the ground truth $\tM$, and algorithm $\{\op_t\}$. Define $\mathbb{E}^{\tM}$ and $\overline{\mathbb{E}}^{\tM}$ to be the respective expectation operators. Since the regret is strictly increasing, if for some $\mu$ close to $\Delta$ we have
\begin{align}\label{eq:goal}
    \overline{\mathbb{P}}^{\bM}\left[\frac{1}{T}\sum_{t = 1}^{\tau}g_M(\pi_t) \geq \mu\right] \geq \Omega(1),
\end{align}
then it will follow that 
\begin{align}
    \mathbb{P}^M\left[\frac{1}{T}\mathsf{Reg}(T) \geq \mu\right] &= \mathbb{P}^{M}\left[\frac{1}{T}\sum_{t = 1}^{T}g_M(\pi_t) \geq \mu\right] \\
    &\geq \mathbb{P}^{\bM}\left[\frac{1}{T}\sum_{t = 1}^{\tau}g_M(\pi_t)\geq \mu\right] - D_{TV}(\overline{\mathbb{P}}^{\bM}, \overline{\mathbb{P}}^{M}).
\end{align}
Crucially, observe that it suffices to bound the total variation distance between $\overline{\mathbb{P}}^{\bM}$ and $\overline{\mathbb{P}}^{M}$ (as opposed to between $\overline{\mathbb{P}}^{\bM}$ and $\mathbb{P}^{M}$), because the events in question are only determined by events up to time $\tau$, and $\overline{\mathbb{P}}^{M}$ and $\mathbb{P}^{M}$ are identical up to time $\tau$. By standard argument (eg. using Pinkser's inequality), one can use the square error constraint of the DEC to bound the total variation distance  $D_{TV}(\overline{\mathbb{P}}^{\bM}, \overline{\mathbb{P}}^{M})$ by an arbitrarily small constant. Thus it suffices to show Equation~\eqref{eq:goal} for the desired value $\mu$.

To prove Equation~\eqref{eq:goal}, we will define the stopping time $\tau$ to be the first step $t$ at which $\sum_{s = 1}^t g_{\bM}(\pi_s) \geq a T$ for some $a < \Delta$. Let $M$ be the model in \eqref{eq:def_dec} that maximizes the DEC for $p=\op_{\bM}$. The key observation is that 
\begin{align}\label{eq:key}
    \frac{1}{T}\sum_{t = 1}^{T} g_M(\pi_t) \approx \frac{1}{T}\sum_{t = 1}^{T} g_{\bM}(\pi_t) + f_{M}(\pi_{M}) - f_{\bM}(\pi_{\bM}).
\end{align}
Such an equivalence holds (with high probability) because the square error constraint in the definition of DEC ensures that most of the time, $\sum_t |f_{\bM}(\pi_t) - f_{M}(\pi_t)|$ is small.
We obtain two observations from this equation. First, on the event that $\frac{1}{T}\sum_{t = 1}^{T} g_{\bM}(\pi_t) \in [a', a)$, we have $\tau = T$, and thus
\begin{align}
    \frac{1}{T}\sum_{t = 1}^{\tau} g_M(\pi_t) \gtrapprox a' + f_{M}(\pi_{M}) - f_{\bM}(\pi_{\bM}).
\end{align} 
Secondly, we can show that the gap $f_{M}(\pi_{M}) - f_{\bM}(\pi_{\bM})$ is large, in particular, nearly as large as $\Delta - a$. This holds because, taking the expectation of Equation~\eqref{eq:key} under $\overline{\mathbb{P}}^{\bM}$, we have
\begin{align}
    \overline{\mathbb{E}}^{\bM}\left[\frac{1}{T}\sum_{t = 1}^{T} g_M(\pi_t)\right] &\lessapprox  \overline{\mathbb{E}}^{\bM}\left[\frac{1}{T}\sum_{t = 1}^{T} g_{\bM}(\pi_t)\right] + f_{M}(\pi_{M}) - f_{\bM}(\pi_{\bM})\\
    &\lessapprox a + f_{M}(\pi_{M}) - f_{\bM}(\pi_{\bM}).
\end{align}
Indeed, by definition of the algorithm $\op$, $\frac{1}{T}\sum_{t = 1}^{T} g_{\bM}(\pi_t)$ never exceeds $a$ (by more than a smidge). Thus since by definition of the DEC, we have $\overline{\mathbb{E}}^{\bM}\left[\frac{1}{T}\sum_{t = 1}^{T} g_M(\pi_t)\right] = \mathbb{E}_{\pi \sim \op_{\bM}}[g_{M}(\pi)] \geq \Delta$, it follows that the gap $f_{M}(\pi_{M}) - f_{\bM}(\pi_{\bM}) \gtrapprox \Delta - a$. 

Combining these observations, we have that 
for any $0 \leq a' < a < \Delta$,
\begin{align}\label{eq:summary}
    \overline{\mathbb{P}}^{\bM}\left[\frac{1}{T}\sum_{t = 1}^{\tau} g_M(\pi_t) \gtrapprox \Delta - a + a'\right] \geq  \mathbb{P}^{\bM}\left[\frac{1}{T}\sum_{t = 1}^{T} g_{\bM}(\pi_t) \in [a', a) \right],
\end{align} 
It remains to choose the values of $a$ and $a'$ to apply this approach to. If $a - a'$ is small, then the regret guarantee will be good, but the probability on the right hand side may be small. Fortunately however, since $\mathbb{P}^{\bM}\left[\frac{1}{T}\sum_{t = 1}^{T} g_{\bM}(\pi_t) \in [0, 1] \right] = 1$, we can divide up the interval $[0, 1]$ into a constant number of shorter intervals $I_i$, and claim that the probability $\mathbb{P}^{\bM}\left[\frac{1}{T}\sum_{t = 1}^{T} g_{\bM}(\pi_t) \in I_i \right]$ is $\Omega(1)$ for one of them. Formally, for some integer $C$, let
\begin{enumerate}
    \item $I_i = [a_i', a_i) := \left[\frac{i - 1}{C} \min(\Delta, \rho), \frac{i}{C} \min(\Delta, \rho)\right)$ for $i = 1, 2, \ldots, C -1$.
    \item $I_C := \left[\frac{i - 1}{C} \min(\Delta, \rho), 1\right]$
\end{enumerate}
For any of the first $i-1$ intervals, we have $\Delta + a_i' - a_i \geq \Delta - \frac{\min(\Delta,\rho)}{C}$; thus Equation~\eqref{eq:summary} yields the desired regret with probability $\mathbb{P}^{\bM}\left[\frac{1}{T}\sum_{t = 1}^{T} g_{\bM}(\pi_t) \in I_i \right]$ for some model $M$ (which may depend on $a_i$).
To handle the final interval, we will need to prove a modification of Equation~\eqref{eq:summary}, which shows that 
for any $0 \leq a' < a < \Delta$,
\begin{align}
    \overline{\mathbb{P}}^{\bM}\left[\frac{1}{T}\sum_{t = 1}^{\tau} g_M(\pi_t) \gtrapprox \min\left(\Delta + a' - \min(\rho, \Delta), a \right)\right] \geq  \mathbb{P}^{\bM}\left[\frac{1}{T}\sum_{t = 1}^{T} g_{\bM}(\pi_t) \geq a' \right].
\end{align} 
Thus by choosing $a \approx \Delta$, and $a' = \frac{C - 1}{C}\min(\rho, \Delta)$, this shows that we achieve the desired regret with probability $\mathbb{P}^{\bM}\left[\frac{1}{T}\sum_{t = 1}^{T} g_{\bM}(\pi_t) \in I_C \right]$ for some model $M$.

This modification is again derived from Equation~\eqref{eq:key}, but it requires an additional argument using the localization assumption that $f_{M}(\pi_{M}) \geq 1 - \rho$, to obtain that $f_{M}(\pi_{M}) - f_{\bM}(\pi_{\bM}) \gtrapprox \Delta - \rho$. We omit the details.

The full proof of Theorem~\ref{prop:decc_rho} appears in Section~\ref{sec:dec_lb}. 

\section{Technical Overview of Upper Bounds (Theorems~\ref{thm:simpleub} and \ref{thm:ub})}\label{sec:ub}
Before giving an overview of the key ideas in the proofs of our upper bounds, we note the following observation which will simplify our discussion.

\begin{observation}\label{obs}
If $\sum_{t = 1}^T\mathbb{E}_{\pi \sim p_t}(f_{M^*}(\pi) - f_{\bM_t}(\pi))^2 \leq \mathsf{Est}(T, \delta)$, then we can bound the regret as: 
\begin{align}
    \sum_{t = 1}^T \gamma f_{M^*}(\pi_{M^*}) - \mathbb{E}_{\pi_t \sim p_t}f_{M^*}(\pi_t) &\leq \sum_{t = 1}^T \gamma f_{M^*}(\pi_{M^*}) - \mathbb{E}_{\pi_t \sim p_t}f_{\bM_t}(\pi_t) + \mathbb{E}_{\pi_t \sim p_t}|f_{M^*}(\pi_t) - f_{\bM_t}(\pi_t)|\\
    &\leq \sum_{t = 1}^T \gamma f_{M^*}(\pi_{M^*}) - \mathbb{E}_{\pi_t \sim p_t}f_{\bM_t}(\pi_t) + \sqrt{T\mathsf{Est}(T, \delta)}.
\end{align}
Here the second inequality follow from applying Jensen's inequality on the distribution $\frac{1}{T}\sum_{t = 1}^T p_t$.
\end{observation}
Since our estimation oracle (Assumption~\ref{assm:oracle}) guarantees that $\sum_{t = 1}^T\mathbb{E}_{\pi \sim p_t}(f_{M^*}(\pi) - f_{\bM_t}(\pi))^2 \leq \mathsf{Est}(T, \delta)$ with probability $1 - \delta$, this observation shows that it suffices to bound $\sum_{t = 1}^T f_{M^*}(\pi_{M^*}) - \mathbb{E}_{\pi_t \sim p_t}f_{\bM_t}(\pi_t)$.

\subsection{Proof Overview of Theorem~\ref{thm:simpleub}}

Theorem~\ref{thm:simpleub} provides a regret guarantee for Algorithm~\ref{alg:1}. Recall that at each round $t$, Algorithm~\ref{alg:1} queries the estimation oracle to gain a model $\bM_t$, and then plays $\pi_t \sim p_t$, where $p_t$ minimizes $\mathsf{dec}^{\gamma}_{\eps_t}(\bM_t)$. Recall that the guarantee of the $\gamma$-DEC is that if we choose an action from the distribution $p_t$ which minimizes $\mathsf{dec}^{\gamma}_{\eps_t}$, then if $M^* \in \cB_{p_t, \eps_t}(\bM_t)$, we have 
\begin{align}\mathbb{E}_{\pi_t \sim p_t}f_{M^*}(\pi_t) \geq \gamma f_{M^*}(\pi_{M^*}) - \mathsf{dec}_{\eps_t}^{\gamma}.
\end{align}
Thus, if we knew that at each round $t$, we had $M^* \in \cB_{p_t, \eps_t}(\bM_t)$, we would immediately get a bound of $\sum_{t = 1}^T \mathsf{dec}_{\eps_t}^{\gamma}$ on the regret, which would remove the $(1 - \gamma)T$ term in Theorem~\ref{thm:simpleub}. 

Unfortunately, it is not always the case that $M^* \in \cB_{p_t, \eps_t}(\bM_t)$; for the schedule of the $\eps_t = \sqrt{\frac{8\mathsf{Est}(T, \delta)}{t}}$ given in Algorithm~\ref{alg:1}, we may have $M^* \notin \cB_{p_t, \eps_t}(\bM_t)$ for a constant fraction of the rounds $t$. This issue is unavoidable, because the estimator $\bM_t$ can be arbitrarily bad at some rounds without violating Assumption~\ref{assm:oracle}. Thus we need some way of bounding the regret in the rounds where $M^* \notin \cB_{p_t, \eps_t}(\bM_t)$. The main step of our proof is showing that for all $t$, we have 
\begin{align}\label{eq:mbarmax}
f_{\bM_t}(\pi_{\bM_t}) \geq \gamma f_{M^*}(\pi_{M^*}) - \mathsf{dec}_{\eps_{\lfloor{t/4}\rfloor}}^{\gamma} - O\left(\sqrt{\mathsf{Est}(T, \delta)/t}\right).
\end{align} 
Since $\bM_t$ is always in the ball $ \cB_{p_t, \eps_t}(\bM_t)$, the $\gamma$-DEC guarantees that we have good regret relative to $\bM_t$ at each round: 
\begin{align}\label{Mbardec}
    \gamma f_{\bM_t}(\pi_{\bM_t}) - \mathbb{E}_{p_t}f_{\bM_t}(\pi_t) \leq \mathsf{dec}^\gamma_{\eps_t}.
\end{align}
Combining Equations~\eqref{eq:mbarmax} and \eqref{Mbardec} and Observation~\ref{obs} essentially yields Theorem~\ref{thm:simpleub}. The growth condition on the DEC (Assumption~\ref{assm:growth}) helps to bound the sum of the $\mathsf{dec}^\gamma_{\eps_t}$ terms. 
These ideas suffice to prove Theorem~\ref{thm:simpleub}, which is proved formally in Section~\ref{sec:ubproof}. 

Our technique in proving Equation~\eqref{eq:mbarmax} is similar to the approach in \citet{foster2023tight}, where we leverage the fact Assumption~\ref{assm:oracle} guarantees that $\bM_t$ is in the convex hull of a certain refined model class, $\cM_{t - 1}$. Indeed, Assumption~\ref{assm:oracle} guarantees that $M_{t-1}$ only includes models on which the estimation error thus far has been small, that is, $\sum_{s = 1}^{t - 1} \mathbb{E}_{\pi \sim p_s}(f_{\bM_s}(\pi) - f_M(\pi))^2 \leq \mathsf{Est}(T, \delta)$.

Now we can argue that any model in the convex hull of $\cM_{t-1}$ has a large optimum. By the estimation error guarantee, there must exist a large set $S$ of rounds in the range $s \in [t/2, t-1]$ on which $M^* \in \cB_{p_s, \eps_s}(\bM_s)$ for $s \in S$. Construct the distribution
$$\hat{p} := \frac{1}{|S|}\sum_{s \in S}p_s.$$

We can write any $\bM_t \in \mathsf{co}(\cM_{t -1})$ as a distribution $\nu$ over models in $\cM_{t-1}$, that is, $\bM_t = \mathbb{E}_{M \sim \nu}M$. We have
\begin{align}
f_{\widehat{M}_t}(\pi_{\widehat{M}_t}) &\geq \mathbb{E}_{\pi \sim \hat{p}}\left[f_{\widehat{M}_t}(\pi)\right] \\
&= \mathbb{E}_{M \sim \nu}\mathbb{E}_{\pi \sim \hat{p}}\left[f_{M}(\pi)\right]\\
&\geq \mathbb{E}_{\pi \sim \hat{p}}\left[f_{M^*}(\pi)\right] - \mathbb{E}_{M \sim \nu}\mathbb{E}_{\pi \sim \hat{p}}|f_M(\pi) - f_{M^*}(\pi)|.
\end{align}
First, because for all $s \in S$, we have $M^* \in \cB_{p_s, \eps_s}(\bM_t)$, the $\gamma$-DEC guarantees that for the first term, we have 
\begin{align}
    \mathbb{E}_{\pi \sim \hat{p}}\left[f_{M^*}(\pi)\right] &\geq \gamma f_{M^*}(\pi_{M^*}) - \min_{s \in S}\mathsf{dec}^{\gamma}_{\eps_{s}}\\
    &\geq \gamma f_{M^*}(\pi_{M^*}) - \mathsf{dec}^{\gamma}_{\eps_{t/2}}.
\end{align}
Second, because all the $M$ in the support of $\bM_t$ are in $\cM_{t - 1}$, their estimation error during all rounds in $S$ must be less than $\mathsf{Est}(T, \delta)$, and thus we are able to prove that
\begin{align}
    \mathbb{E}_{\pi \sim \hat{p}}|f_M(\pi) - f_{M^*}(\pi)| \leq O\left(\sqrt{\mathsf{Est}(T, \delta)/t}\right).
\end{align}
Putting these two steps together yields Equation~\eqref{eq:mbarmax}.

We formally prove Theorem~\ref{thm:simpleub} in Section~\ref{sec:ubproof}.

\subsection{Proof Overview of Theorem~\ref{thm:ub}}
In this section, we sketch some of the main ideas in the proof of Theorem~\ref{thm:ub}, which analyzes the $\gamma$-regret of Algorithm~\ref{alg:best}. 
Unfortunately, this reduction to calling the guarantee of the DEC on $\bM_t$ instead of $M^*$ leads to an additional $(1 - \gamma)T$ term in the regret, which is problematic in the case when $\gamma < 1$. Thus we need some additional ideas to prove Theorem~\ref{thm:ub}.
The key difference between Algorithm~\ref{alg:best} and Algorithm~\ref{alg:1} is that in Algorithm~\ref{alg:best}, the value $\eps_t$ is chosen in a very careful manner.

The motivation for Algorithm~\ref{alg:best} comes from the idea that we should predict if the estimator $\bM_t$ is ``bad'', and if so, we should choose a distribution which \em explores \em more by choosing a large $\eps_t$, rather than \em exploiting \em by choosing a small $\eps_t$. Exploring more corresponds to choosing a larger value for $\eps_t$, because if $M^* \notin \cB_{p_t, \eps_t}(\bM_t)$, we will have learned a lot about $M^*$.

It turns out that for a given $\eps$, if we know the value of the optimum $f_{M^*}(\pi_{M^*})$, we can predict whether $M^*$ will be excluded from the ball $\cB_{p, \eps}(\bM)$, where $p$ is the distribution that minimizes $\mathsf{dec}_{\eps}^{\gamma}(\bM)$. 
With a further continuity assumption (Assumption~\ref{assm:cont}), we can use this knowledge to choose a good value for $\eps_t$ at each round, and attain an ideal trade-off between exploration and exploitation. To understand this more quantitatively, we define the following function $g$ which will be used in Algorithm~\ref{alg:best} to choose the value of $\eps_t$.
\begin{definition}\label{def:G}
For any $\bM$ and $\eps$, define $p(\eps, \bM)$ to be the distribution guaranteed by Assumption~\ref{assm:cont}. Define the function 
\begin{align}
    g(\eps, \bM, v) := \gamma v - \mathbb{E}_{p(\eps, \bM)}f_{\bM}(\pi) - 2\eps - \Delta_{\eps},
\end{align}
\end{definition}

We can prove the following lemma, which is the key idea in the proof of Theorem~\ref{thm:ub}.

\begin{lemma}\label{lemma:G}
Suppose Assumption~\ref{assm:cont} holds. For any $\bM$ and $v \in [0, 1]$, there exists some $\eps \geq 0$ such that $g(\eps, \bM, v) \leq 0$. Let $\eps^* := \arg \min_{\eps \geq 0} g(\eps, \bM, v) \leq 0$, and let $p^* := p(\eps^*, \bM)$ as in Assumption~\ref{assm:cont}. Then:
\begin{enumerate}
    \item \textbf{Exploitation:} We have $\mathbb{E}_{\pi \sim p*}[f_{\bM}(\pi)] \geq \gamma v - 2\eps - \Delta_{\eps^*}$.
    \item \textbf{Exploration:} For any $M \in \cM$ with $f_M(\pi_M) \geq v$, we have $\mathbb{E}_{p*}(f_{\bM}(\pi) - f_{M}(\pi))^2 \geq (\eps^*)^2$.
\end{enumerate}
\end{lemma}
\begin{proof}
First observe that if $\eps \geq 1/2$, then we will have $g(\eps, \cM, v) \leq 0$. Second observe that by Assumption~\ref{assm:cont}, the function $g(\eps, \bM, v)$ is continuous in $\eps$, and thus this argmin $\eps^*$ exists. In particular, either $\eps^* = 0$, or $g(\eps^*, \bM, v) = 0$.

Now the first item of the Lemma follows from the fact that 
\begin{align}
    \mathbb{E}_{\pi \sim p*}[f_{\bM}(\pi)] &= \gamma v - 2\eps^* - \Delta_{\eps^*} - g(\eps^*, \bM, v),
\end{align}
and $g(\eps^*, \bM, v) \leq 0$. For the second item, observe that if $M \in \cB_{p^*, \eps^*}$, then since $p^*$ bounds the $\gamma$-DEC by $\Delta_{\eps}$, we must have 
\begin{align}
    \gamma f_{M}(\pi_{M}) - \mathbb{E}_{p^*}f_{M}(\pi) \leq \Delta_{\eps^*},
\end{align}
and 
\begin{align}
    \mathbb{E}_{p^*}|f_{\bM}(\pi) - f_{M}(\pi)| \leq \sqrt{\mathbb{E}_{p*}[(f_{\bM}(\pi) - f_{M}(\pi))^2]} \leq \eps^*.
\end{align}
These two together imply that 
\begin{align}
    g(\eps^*, \bM, v) &= \gamma v - \mathbb{E}_{p^*}f_{\bM}(\pi) - 2\eps^* - \Delta_{\eps^*}\\
    &\leq \gamma v - \mathbb{E}_{p^*}f_{M}(\pi) - \eps^* - \Delta_{\eps^*}\\
    &\leq \gamma(v - f_M(\pi_M)) - \eps^*\\
    &\leq -\eps^*.
\end{align}
If $\eps^* > 0$, then this would imply that $g(\eps^*, \bM, v) < 0$, which contradicts that fact that $g(\eps^*, \bM, v) = 0$ exactly. If $\eps^* = 0$, it is immediate that $\mathbb{E}_{p*}(f_{\bM}(\pi) - f_{M}(\pi))^2 \geq (\eps^*)^2$. Thus in either case we have $\mathbb{E}_{p*}(f_{\bM}(\pi) - f_{M}(\pi))^2 \geq (\eps^*)^2$ as desired.
\end{proof}

\begin{algorithm}[t]
\caption{Simplified version of Algorithm~\ref{alg:best}}\label{alg:improved}
\textbf{input:}  Parameter $\gamma \in [0, 1]$, the optimum value $f_{M^*}(\pi_{M^*})$.
\begin{algorithmic}[0]
    \For{$t=1,\ldots,T$}
    \State Let $\bM_t = \mathsf{AlgEst}(\mathcal{H}_{t - 1})$.
    \State Let $\eps_t := \min_{\eps}: g(\eps, \bM_t, f_{M^*}(\pi_{M^*})) \leq 0$, where $g$ is defined in Definition~\ref{def:G}.
   \State Play $\pi_t \sim p_t$ where $p_t := p(\eps_t, \bM_t)$ is the distribution guaranteed by Assumption~\ref{assm:cont}.
    \EndFor
\end{algorithmic}
\end{algorithm}

To understand how to leverage Lemma~\ref{lemma:G}, we will analyze in this section a simplified setting where we assume the algorithm has knowledge of $f_{M^*}(\pi_{M^*})$. Of course in Algorithm~\ref{alg:best}, we do not know this maximum, so we will need to use some proxy for it. We state the simplification of Algorithm~\ref{alg:best} in Algorithm~\ref{alg:improved}. In Algorithm~\ref{alg:improved}, at each round $t$, we obtain from the estimation oracle an estimate $\bM_t$ of the model. We then choose $\eps_t$ as in Lemma~\ref{lemma:G} to be the minimum value of $\eps$ such that $g(\eps, \bM_t, f_{M^*}(\pi_{M^*})) \leq 0$. Finally, we play $p_t := p(\eps_t, \bM_t)$ to be the distribution from Assumption~\ref{assm:cont} that upper bounds the $\gamma$-DEC.

By playing this $p_t$, Lemma~\ref{lemma:G} (with $v = f_{M^*}(\pi_{M^*})$) guarantees that $M^* \notin \cB^o_{p_t, \eps_t}(\bM_t)$, where we define the \em open \em ball $\cB^o_{p, \eps}(\bM) := \{M : \mathbb{E}_{p}(f_{\bM}(\pi) - f_{M}(\pi))^2 < \eps^2\}$. While it may seem counter intuitive that we want to choose $(\eps_t, p_t)$ such that we have $M^* \notin \cB^o_{p_t, \eps_t}$ the first bullet of Lemma~\ref{lemma:G} guarantees that we get as good of a regret guarantee for this $\eps_t$ as we would if we have $M^* \in \cB_{p_t, \eps_t}$. Because $M^*$ is \em excluded \em from the $L_2$ ball, by the guarantee of our estimation oracle, the $\eps_t$ cannot be too large at each round, that is, we must have with probability $1 - \delta$,
\begin{align}
    \sum_{t = 1}^T \eps_t^2 \leq \sum_{t = 1}^T \mathbb{E}_{\pi_t \sim p_t}(f_{\bM}(\pi_t) - f_{M^*}(\pi_t))^2 \leq \mathsf{Est}(T, \delta).
\end{align}

Putting together Lemma~\ref{lemma:G} and this bound on $\sum_{t = 1}^T \eps_t^2$, we can prove the following bound on the regret of Algorithm~\ref{alg:improved}.

\begin{proposition}\label{thm:ubprop}
Suppose Assumptions~\ref{assm:oracle}, \ref{assm:growth} and \ref{assm:cont} hold, such that $\mathsf{dec}^{\gamma}_{\eps} \leq \Delta_{\eps}$. Suppose additionally that the value $f_{M^*}(\pi_{M^*})$ is known. Then with probability $1 - \delta$, Algorithm~\ref{alg:improved} attains the regret

\begin{align}
    \mathsf{Reg}_{\gamma}(T) \leq T\Delta_{\sqrt{\mathsf{Est}(T, \delta)/T}}+ 3\sqrt{T\mathsf{Est}(T, \delta)}.
\end{align}
\end{proposition}
\begin{proof}
Condition on the event that 
\begin{align}
    \sum_{t = 1}^T \mathbb{E}_{\pi \sim p_t}(f_{M^*}(\pi) - f_{\bM_t}(\pi))^2 \leq \mathsf{Est}(T, \delta),
\end{align}
which occurs with probability $1 - \delta$. Considering the total regret over $T$ rounds and applying Observation~\ref{obs}, we have:
\begin{align}
    \sum_{t = 1}^T  \gamma f_{M^*}(\pi_{M^*}) - \mathbb{E}_{p_t}f_{M^*}(\pi) &\leq \sum_{t = 1}^T  \gamma f_{M^*}(\pi_{M^*}) - \mathbb{E}_{p_t}f_{\bM_t}(\pi) + \sqrt{T\mathsf{Est}(T, \delta)}.
\end{align}

By item 1 of Lemma~\ref{lemma:G}, applied with $v = f_{M^*}(\pi_{M^*})$, we have 
\begin{align}
    \gamma f_{M^*}(\pi_{M^*}) - \mathbb{E}_{p_t}f_{\bM_t}(\pi) \leq 2\eps_t - \mathsf{dec}^\gamma_{\eps_t}.
\end{align}

Now, by item 2 of Lemma~\ref{lemma:G}, at each step $t$, we have $\mathbb{E}_{p(\eps_t, \bM_t)}[(f_{\bM}(\pi) - f_{M^*}(\pi))^2] \geq \eps_t^2$. Thus we must that $\sum_{t = 1}^T \eps_t^2 \leq \sum_{t = 1}^T \mathbb{E}_{p(\eps_t, \bM_t)}[(f_{\bM}(\pi) - f_{M^*}(\pi))^2] 
 \leq \mathsf{Est}(T, \delta)$. Thus our final regret is at most:
\begin{align}
    \sum_{t=1}^T 2\eps_t + \mathsf{dec}^\gamma_{\eps_t} + \sqrt{T\mathsf{Est}(T, \delta)} &\leq \sup_{x \in \mathbb{R}^T, \|x\|^2 \leq \mathsf{Est}(T, \delta)} \sum_i \left(2x_i + \mathsf{dec}^{\gamma}_{x_i}\right) + \sqrt{T\mathsf{Est}(T, \delta)}\\
    &\leq \sup_{x \in \mathbb{R}^T, \|x\|^2 \leq \mathsf{Est}(T, \delta)} \sum_i \mathsf{dec}^{\gamma}_{x_i} + 3\sqrt{T\mathsf{Est}(T, \delta)}.
\end{align}
Applying Holder's inequality (see Lemma~\ref{lemma:holder}) yields that under Assumption~\ref{assm:growth}, we have 
\begin{align}\sup_{x \in \mathbb{R}^T, \|x\|^2 \leq \mathsf{Est}(T, \delta)} \sum_i \mathsf{dec}^{\gamma}_{x_i} \leq T\Delta_{\sqrt{\frac{\mathsf{Est}(T, \delta)}{T}}},
\end{align}
which yields the proposition.
\end{proof}

The proof of Theorem~\ref{thm:ub} and the statement of Algorithm~\ref{alg:best} are given in Section~\ref{sec:ubproof}. 

\section{Comparison to other works on the DEC.}\label{sec:rw}
\paragraph{Other definitions of the DEC.} The work of \citet{foster2023tight}, which defined the ``constrained'' DEC  (which the $\gamma$-DEC is based on) follows the earlier work of \citet{foster2021statistical}, which originally introduced the notion of a decision-estimation coefficient. This original work defined the DEC (or \em offset \em DEC) as $ \mathsf{dec}_{\lambda} := \min_p \max_M f_M(\pi_M) - f_M(\pi) - \lambda \mathbb{E}[D_H^2(M(\pi), \bM(\pi))]$, where $D_H$ is the Hellinger distance. This offset DEC yielded bounds of the form $\Theta\left(\inf_{\lambda}(T\mathsf{dec}_{\lambda} + \lambda)\right) \leq \mathbb{E}[\mathsf{Reg}(T)] \leq \Theta\left(\inf_{\lambda}(T\mathsf{dec}_{\lambda} + \lambda\log(|\cM|))\right)$, though under much more stringent conditions on the class $\cM$ for the lower bound. 

\paragraph{Extension to Decision Making with Structured Observations}\label{rem:hellinger}
For simplicity, we stated Theorem~\ref{prop:decc_rho} for bandits with Gaussian noise, but our lower bound can be extended to the more general setting of Decision Making with Structured Observations (see \citet{foster2021statistical, foster2023tight}; this also encompasses reinforcement learning and MDPs). If the density ratios $\frac{M(\pi)}{\bM(\pi)}$ are bounded, then we achieve the same result as Theorem~\ref{prop:decc_rho} if we define the $\gamma$-DEC using a constraint on the squared Hellinger Distance, $\mathbb{E}_{\pi \sim p}D_{H}^2(\bM(\pi), M(\pi)) \leq \eps^2$ instead of the squared error to $f_{\bM}$ constraint. If the density ratios are unbounded, the same result as Theorem~\ref{prop:decc_rho} holds, using the $\gamma$-DEC with the squared Hellinger distance constraint, under a slightly stronger condition that $\eps \leq \frac{1}{\sqrt{T}\log(T)}$. (See \citet[Lemma A.13]{foster2021statistical}). Our upper bounds (Theorems~\ref{thm:simpleub} and \ref{thm:ub}) can be similarly extended to the general Decision Making with Structured Observations setting by using a definition of the $\gamma$-DEC with a squared Hellinger constraint. 

We remark that with this Hellinger constraint modification, for the case that $\gamma = 1$, our Theorem~\ref{thm:simpleub} recovers the same result as the regret upper bound in \citet{foster2023tight}. To simplify the exposition in our paper, we have strengthened two of the assumptions in their paper (the regularity condition in Assumption~\ref{assm:growth} is stronger than the regularity assumption they assume), and we have assumed a slightly different estimation oracle (as per Assumption~\ref{assm:oracle}) instead of their constrained estimation oracle assumption. However, our Algorithm~\ref{alg:1} can be (very minorly) modified to be run in epochs of doubling size to work with their oracle, and our proof can be (very slightly) modified to hold under their regularity condition. Making these changes recovers their result with a much simpler algorithm.

\section{Examples}\label{sec:ex}

\subsection{Example of $1$-DEC for the Multi-armed Bandit (MAB).}\label{sec:ex_mab}
In the following proposition, we construct a distribution $p$ that near-optimally upper bounds $1$-DEC in the setting where $\cM$ is the class of MAB problems with $K$-arms and mean rewards in $[0, 1]$. We provide this example to show that in this simple setting, it is possible to satisfy Assumption~\ref{assm:cont} by upper bounding the $\mathsf{dec}_{\eps}^{\gamma}$ by a distribution $p$ that is continuous in $\eps$.
\begin{proposition}\label{prop:mab}
Fix any $\eps > 0$ and reference model $\bM$ on $k$-arms with mean reward $f_{\bM}(i)$ for $i \in [K]$. Define the distribution $p$ on $[K]$ as 
\begin{align}
    p(i) = \frac{\eps^2}{(\lambda + \max _j f_{\bM}(j) - f_{\bM}(i))^2},
\end{align}
where $\lambda$ is chosen such that $\sum_i p(i) = 1$.
Then for any $M \in \cM \cup \{\bM\}$ such that $\mathbb{E}_{i \sim p}(f_{\bM}(i) - f_{M}(i))^2 \leq \eps^2$, we have
\begin{align}
    \mathbb{E}_{i \sim p}[f_{M}(\pi_M) - f_{M}(i)] \leq (\sqrt{K} + 1)\eps.
\end{align}
\end{proposition}
This distribution near-optimally bounds the DEC for the $K$-armed MAB. Indeed, it is straightforward to check that the $\mathsf{dec}^1(\cM)$ grows roughly like $\eps \sqrt{K}$:
\begin{lemma}\label{lemma:mablb}
For the $K$-armed MAB model class $\cM$, if $K \geq 2$ and $\eps \leq 1/2$, we have 
\begin{align}
    \mathsf{dec}_{\eps}^1(\cM) \geq \min\left(\eps\sqrt{K-1}, 1/4\right).
\end{align}
\end{lemma}
\begin{proof}[Proof of Proposition~\ref{prop:mab}]
First observe that if $\mathbb{E}_{i \sim p}(f_{\bM}(i) - f_{M}(i))^2 \leq \eps^2$, we have by Jenson's inequality $\mathbb{E}_{i \sim p}|f_{\bM}(i) - f_{M}(i)| \leq \eps$, and so
\begin{align}\label{eq:mab}
    \mathbb{E}_{i \sim p}[f_{M}(\pi_M) - f_{M}(i)] \leq \mathbb{E}_{i \sim p}[f_{M}(\pi_M) - f_{\bM}(i)] + \eps.
\end{align}

Further, letting $i^* := \arg \max_i f_{\bM}(i)$, since $\mathbb{E}_{i \sim p}(f_{\bM}(i) - f_{M}(i))^2 \leq \eps^2$, for all $i$, we must have 
\begin{align}
    |f_M(i) - f_{\bM}(i)| \leq \frac{\eps}{\sqrt{p(i)}} = \lambda + f_{\bM}(i^*) - f_{\bM}(i),
\end{align}
and so 
\begin{align}
f_{M}(\pi_M) \leq f_{\bM}(\pi_M) + \lambda + f_{\bM}(i^*) - f_{\bM}(\pi_M) = \lambda + f_{\bM}(i^*).
\end{align}

Now we can compute 
\begin{align}
    \mathbb{E}_{i \sim p}[f_{M}(\pi_M) - f_{\bM}(i)] &= \sum_{i} p(i) \left( f_{M}(\pi_M) - f_{\bM}(i)\right) \\
    &\leq \sum_{i} p(i) \left(\lambda + f_{\bM}(i^*) - f_{\bM}(i)\right)\\
    &= \eps^2 \sum_{i} \frac{1}{\lambda + f_{\bM}(i^*) - f_{\bM}(i)}
\end{align}

Now we know $\lambda$ has to be large enough such that $\sum_i \frac{\eps^2}{(\lambda + f_{\bM}(i^*) - f_{\bM}(i))^2} \leq 1$.

Reparameterize such that $b_i := \lambda + f_{\bM}(i^*) - f_{\bM}(i)$. Then we need to solve the following optimization problem over $\{b_i \geq 0\}_{i \in [K]}$:
\begin{align}
    \max \sum_i \frac{1}{b_i}, \text{ such that } \: \sum_i \frac{\eps^2}{b_i^2} \leq 1.
\end{align}
The maximum is attained when all $b_i = \eps\sqrt{K}$. Thus we have 
\begin{align}
    \mathbb{E}_{i \sim p}[f_{M}(\pi_M) - f_{\bM}(i)] \leq \eps^2 \sum_i \frac{1}{b_i} = \eps \sqrt{K}.
\end{align}
Combining with Equation~\eqref{eq:mab} yields the proposition.

\end{proof}
\begin{proof}[Proof of Lemma~\ref{lemma:mablb}]
Let $\bM$ be the model with $f_{\bM}(i) = \frac{\eps}{\sqrt{K - 1}}$ for all $i \in [K]$. Fix any distribution $p$ on $[K]$, and denote $p_i := p(i)$ for $i \in [K]$. Let $j := \arg \min_{i \in [K]} p_i$, such that $p_j \leq 1/K$. Let $M$ be the model with 
\begin{align}
    \begin{cases}
        f_M(i) = \min\left(1, \frac{\eps}{\sqrt{K - 1}} + \eps \sqrt{\frac{1 - p_j}{p_j}}\right) & i = j\\
        f_M(i) = \frac{\eps}{\sqrt{K - 1}} - \eps \sqrt{\frac{p_j}{1 - p_j}} & i \neq j,
    \end{cases}
\end{align}
such that $f_M(i) \in [0, 1]$ for all $i$.

Then we have 
\begin{align}
\mathbb{E}_{i \sim p}(f_{\bM}(i) - f_M(i))^2 = (1 - p_j)\eps^2 \frac{p_j}{1 - p_j} + p_j \eps^2 \frac{1 - p_j}{p_j} = \eps^2,
\end{align} and 
\begin{align}
    f_M(j) - \mathbb{E}_p f_M(j) = (1 - p_j)\left(\min\left(1 - \frac{\eps}{\sqrt{K - 1}}, \eps \sqrt{\frac{1-p_j}{p_j}}\right)+ \eps \sqrt{\frac{p_j}{1-p_j}}\right).
\end{align}
If $\frac{\eps}{\sqrt{K - 1}} + \eps \sqrt{\frac{1 - p_j}{p_j}} \leq 1$, then the regret under $p$ equals
\begin{align}
(1 - p_j)\eps\left(\sqrt{\frac{1-p_j}{p_j}} + \sqrt{\frac{p_j}{1-p_j}}\right) = \eps \sqrt{\frac{1-p_j}{p_j}} \geq \eps \sqrt{K - 1}.
\end{align}
Otherwise the regret under $p$ is at least $(1 - p_j)(1- \frac{\eps}{\sqrt{K-1}}) \geq 1/4$ since $\eps \leq 1/2$.

\end{proof}

\subsection{Example of $\gamma$-Regret bounds via the $\gamma$-DEC}
In this section, we give an example of a bandit problem for which whenever $T$ is sub-exponential in the dimension of the action space, we can prove nearly matching upper and lower bounds on the $\gamma$-regret. This example is constructed by linearly combining two problems in the proportions $\gamma$ and $1 - \gamma$: a standard K-armed MAB, and a bandit problem where the feedback is given by a shifted ReLU function of $w^T\pi$, where $w \in \mathbb{R}^d$ is the ground truth, and $\pi \in \mathbb{R}^d$ is the action. While the $K$-MAB problem is easy to solve with $\sqrt{KT}$ regret, the shifted ReLU problem cannot be solved with less than exponentially in $d$ many steps, and thus the regret is $\exp(\Theta(d))\sqrt{T}$. If we constrain $T$ to be less than $\exp(\Theta(d))$, then we will see that the $\gamma$-regret of the example is equal to the regret of the $K$-MAB problem.

While this example is artificial, there may be more natural settings where a similar phenomena occur where some part of the function is easy to maximize, and another part is not. For example, \cite{rajaraman2023beyond} studies a setting where there is an initial ``burn-in'' phase where the regret is large, and then a second learning phase where the regret grows less quickly as a function of $T$. This parallels our setting where initially we can only achieve a sublinear regret against a $\gamma$-benchmark, but when $T$ becomes exponentially large, we can achieve sublinear regret. 

\begin{example}\label{ex}
For any $d, K$, consider the following model class $\cM$ with the action set $\Pi := \{(\pi', \pi) \in \{-1, 1\}^d \cup \bot \times [K]\}$. Let $\cM$ be parameterized by $\{(w, v) \in \{-1, 1\}^d \times [0, 1]^K\}$, and define
\begin{align}
    f_{(w, v)}(\pi', \pi) = (1 - \gamma)h_{w}(\pi') + \gamma g_v(\pi), 
\end{align}
where $g_{v}(\pi) = v_{\pi}$ and
\begin{align}
    h_{w}(\pi') = \begin{cases}
        \frac{2}{d}\max\left(0, \langle\pi',w\rangle - \frac{d}{2}\right)& \pi' \neq \bot\\
        0 & \pi' = \bot.
    \end{cases}
\end{align}
\end{example}

It is straightforward to see that the function $h_{w}$ is hard to maximize in less that $\exp(\Theta(d))$ queries, since finding any $\pi' \in \{-1,  1\}^d$ with $\langle \pi', w \rangle > \frac{d}{2}$ with high probability requires $\exp(\Theta(d))$ queries. However, it is possible to maximize $g_v$ by any standard MAB algorithm. Thus if we find $\text{argmax} \: v$ but not $w$, we can achieve $\gamma$ of the optimum, which makes $\gamma$ a natural approximation ratio for this model class. 

We proceed to prove upper and lower bounds on the $\gamma$-regret for this model class. The bounds are matching up to $\log(K)$ factors.

\begin{proposition}[Upper bound for Example~\ref{ex}]
We have $\mathsf{dec}^{\gamma}_{\eps} \leq \gamma (\sqrt{K} + 1)\eps$.
Further, for some algorithm, for any $T \geq 1$, we have $\mathsf{Reg}_{\gamma}(T) \leq 6\sqrt{KT\log(K)}$ with probability at least $1 - \frac{1}{k}$.
\end{proposition}
\begin{proof}
We first bound the $\gamma$-DEC. For any reference model $\bM$, let $\bM_1$ be a model over the action space $[K]$ given by $\bM_1(\pi) = \bM(\perp, \pi)$. 

For any $\eps$, let $p = p(\eps, \bM)$ be the distribution on $\Pi$ which returns a sample $(\perp, \pi)$, where $\pi \sim q(\eps, \bM_1)$, where $q(\eps, \bM_1)$ is any distribution upper bounding the $\mathsf{dec}_{\eps}^1$ by $(\sqrt{K} + 1)\eps$ for the $K$-MAB problem. Such a distribution exists by Proposition~\ref{prop:mab}.

Then for any $M \in \cB_{p, \eps}(\bM)$ with mean $f_{w, v}$, we have 
\begin{align}
    \gamma f_{w, v}((w, v)) - \mathbb{E}_{\pi', \pi \sim p}[f_{w, v}(\pi', \pi)] &\leq \gamma(1 - \gamma)h_w(w) + \gamma^2 g_v(v) - \gamma \mathbb{E}_{\pi \sim q}[g_v(\pi)]\\
    &= \gamma(1 - \gamma)\left(f_{w, v}((w, v))-g_v(v)\right) + \gamma \left(g_v(v) - \mathbb{E}_{\pi \sim q}[g_v(\pi)]\right)\\
    &= \gamma \left(g_v(v) - \mathbb{E}_{\pi \sim q}[g_v(\pi)]\right)\\
    &\gamma (\sqrt{K} + 1)\eps,
\end{align}
where the last line follows because the MAB model with mean $g_v$ must be inside $\cB_{q, \eps}(\bM_1)$, otherwise it would be impossible to have $M \in \cB_{p, \eps}(\bM)$.

To bound $\mathsf{Reg}_{\gamma}(T)$, we could apply Theorem~\ref{thm:ub} with an estimation oracle satisfying Assumption~\ref{assm:oracle} with $\mathsf{Est}(T, \delta) = \tilde{O}(K)$ (see \citet{foster2021statistical}, which guarantees such an oracle), which would yield a regret guarantee of $\tilde{O}(K\sqrt{T})$.

However, since this is not tight, instead let $\{p_t\}$ be any MAB algorithm which achieves a regret of at most $6\sqrt{KT\log(K)}$ with probability at least $1 - \frac{1}{K}$. Such an algorithm exists, see eg. the EX3.P algorithm in \citet[Theorem 3.3]{bubeck2012regret}.

Consider the following algorithm, which at round $t$,  chooses $\pi_t \sim p_t$, and $\pi'_t = \bot$. Since there is no feedback from the $h_w$ part of the function since the algorithm always chooses $\bot$, the $\gamma$-regret at each step will be the same as the $\gamma$ times the regret of the MAB algorithm the on function class given by $\{g_v\}_{v \in [0, 1]^k}$. This yields the proposition.
\end{proof}

\begin{proposition}[Lower Bound for Example~\ref{ex}]
For any $T \leq \exp(d/8)$, there exists some model class $\tilde{\cM} \subset \cM$, such that $\mathsf{dec}^{\gamma}_{\eps}(\tilde{\cM}) \geq \frac{1}{2\sqrt{20}}\sqrt{\frac{K}{T}}$ for $\eps = \frac{1}{3\sqrt{T}}$.

Further, for $K \geq 6000/\gamma^2$, $\tilde{\cM}$ is $6\eps$-localized, and thus by Theorem~\ref{prop:decc_rho}, for some $M \in \cM$,  with probability at least $\frac{1}{3}$, we have $\mathsf{Reg}_{\gamma}(T) \geq \frac{\gamma}{10}\sqrt{KT}$.
\end{proposition}
\begin{proof}
First we define the subset $\tilde{\cM} \subset \cM$ to be the set of all models indexed by $w, v$ such that $v$ has exactly one coordinate equal to $1$, and the remaining coordinates all equal $1 - \sqrt{\frac{K}{20T}}$. For $k \in [K]$, we will use the shorthand $f_{w, k}$ to denote the function $f_{w, v}$, where $v$ is the vector with a $1$ in coordinate $k$, and $1 - \sqrt{\frac{K}{20T}}$ in all other coordinates.

Define $\widehat{f}(\pi', \pi) := \gamma\left(1 - \sqrt{\frac{K}{20T}}\right)$, and let $\bM$ be the corresponding model $\bM(\pi', \pi) = \mathcal{N}(\widehat{f}(\pi', \pi), 1)$, which denotes the normal distribution with mean $\widehat{f}(\pi', \pi)$. Fix any distribution $p$ on $\Pi$. Let
\begin{align}\tilde{w} := \text{argmin}_{w \in \{-1, 1\}^d} \mathbb{E}_{(\pi', \pi) \sim p}\left[\mathbf{1}\left(\langle w, \pi' \rangle \geq \frac{d}{2}\right)\right],
\end{align} and let 
\begin{align}\label{eq:best_M}\tilde{k}:= \text{argmin}_{k \in [K]} \mathbb{E}_{(\pi', \pi) \sim p}\left[\mathbf{1}(\pi = k) \right].
\end{align}

Observe that
\begin{align}
\mathbb{E}_{(\pi', \pi) \sim p}\left[\mathbf{1}(\pi = \tilde{k})\right] &\leq \mathbb{E}_{k \sim \text{Uniform}[K])} \mathbb{E}_{(\pi', \pi) \sim p}\left[\mathbf{1}(\pi = k)\right]\\
&=  \mathbb{E}_{(\pi', \pi) \sim p}\mathbb{E}_{k \sim \text{Uniform}([K])}\left[\mathbf{1}(\pi = k)\right]\\
&= \frac{1}{K}.
\end{align}
Similarly, 
\begin{align}
\mathbb{E}_{(\pi', \pi) \sim p}\left[\mathbf{1}(\langle \tilde{w}, \pi' > d/2\rangle) \right] &\leq \mathbb{E}_{w \sim \text{Uniform}(\{-1, 1\}^d)} \mathbb{E}_{(\pi', \pi) \sim p}\left[\mathbf{1}(\langle w, \pi' > d/2\rangle) \right]\\
&=  \mathbb{E}_{(\pi', \pi) \sim p}\mathbb{E}_{w \sim \text{Uniform}(\{-1, 1\}^d)}\left[\mathbf{1}(\langle w, \pi' > d/2\rangle)\right]\\
&\leq \mathbb{P}\left[\text{Bin}\left(d, \frac{1}{2}\right) \geq \frac{3}{4}d\right]\\
&\leq \exp(-d/8),
\end{align}
where the last line follows from Hoeffding's inequality.

Let $M$ be the model with mean $f_{\tilde{w}, \tilde{k}}$. Then we have
\begin{align}
\mathbb{E}_{(\pi', \pi) \sim p}\left[\left(f_M(\pi', \pi) - f_{\bM}(\pi', \pi)\right)^2\right] &\leq  2\mathbb{E}_{(\pi', \pi) \sim p}\left[\mathbf{1}\left(\langle \tilde{w}', \pi' \rangle \geq \frac{d}{2}\right) \right] + 2\mathbb{E}_{(\pi', \pi) \sim p}\left[\left(\sqrt{\frac{K}{20T}}\right)^2\mathbf{1}\left(\pi = \tilde{k}\right) \right]\\
&\leq 2\exp(-d/8) + \frac{2}{20T} \leq \eps^2.
\end{align}

Now we consider the $\gamma$-regret of $f_M$ under $p$:
\begin{align}
    \mathbb{E}_{(\pi', \pi) \sim p}\left[\gamma f_M^* - f_M(\pi', \pi)\right] &\geq \gamma - \mathbb{E}_{(\pi', \pi) \sim p}\left[f_M(\pi', \pi)\right]\\
    &\geq \gamma - \mathbb{P}_{(\pi', \pi) \sim p}\left[\mathbf{1}(\langle \tilde{w}, \pi' > d/2\rangle)\right] - \gamma\left(1 - \sqrt{\frac{K}{20T}}\right) - \gamma\sqrt{\frac{K}{20T}} \mathbb{P}_{(\pi', \pi) \sim p}\left[\mathbf{1}(\tilde{k} = \pi)\right]\\
    &\geq -\exp(-d/8) + \gamma\sqrt{\frac{K}{20T}}\left(1 - \mathbb{P}_{(\pi', \pi) \sim p}\left[\mathbf{1}(\tilde{k} = \pi)\right]\right)\\
    &\geq -\exp(-d/8) + \gamma\sqrt{\frac{K}{20T}}\left(1 - \frac{1}{K}\right)\\
    &\geq  \frac{\gamma}{2}\sqrt{\frac{K}{20T}},
\end{align}
where in the final line we used the fact that $T \leq \exp(d/8)$.

It follows that $\mathsf{dec}^{\gamma}_{\eps}(\tilde{\cM}) \geq \frac{\gamma}{2}\sqrt{\frac{K}{20T}}$.

Finally, we check that we can apply Theorem~\ref{prop:decc_rho} with $\eps = \frac{1}{3\sqrt{T}}$ and $C = 1$. We observe that $\tilde{\cM}$ is $6\eps$-localized (indeed, all of its maxima are identical), and that $\mathsf{dec}^{\gamma}_{\eps}(\cM) \geq \frac{\gamma}{2}\sqrt{\frac{K}{20T}} \geq 6\eps$ for $K \geq 6000/\gamma^2$. Thus it follows that for some $M \in \cM$, $\mathbb{P}^M\left[\frac{1}{T}\mathsf{Reg}_{\gamma}(T) \geq \frac{\gamma}{2\sqrt{20}}\sqrt{\frac{K}{T}} - 18\eps\right] \geq \frac{1}{3}$, so for $K \geq \frac{6000}{\gamma^2}$, we have $\mathbb{P}^M\left[\mathsf{Reg}_{\gamma}(T) \geq \frac{\gamma}{10}\sqrt{KT}\right] \geq \frac{1}{3}$.
\end{proof}

\subsection{Example of $\gamma$-DEC with Approximate Linear Optimization Oracles.}\label{sec:apxoracles}
Inspired by the works \citet{kakade2007playing, garber2017efficient, hazan2018online}, which study a linear minimization setting with access to an $\alpha$-approximate minimization oracle over the space $\Pi$, we give an example here of a comparable linear maximization setting where we have an access to an oracle which attains a $\gamma$-approximate maximum over the space $\Pi$.

We use this oracle and a barycentric spanner of $\Pi$ (which is used in \citet{kakade2007playing, garber2017efficient, hazan2018online}) to provide an upper bound scaling with $\eps^{2/3}$ on $\mathsf{dec}_{\eps}^{\gamma}$. Assuming a regression oracle for the linear model class satisfying Assumption~\ref{assm:oracle}, Theorem~\ref{thm:ub} yields a $\gamma$-regret scaling with $T^{2/3}$, which mirrors the $T^{2/3}$ regret bound in \citet{kakade2007playing, garber2017efficient, hazan2018online}.

Formally, let $\Pi \subset \mathbb{R}^d$, and suppose the mean reward function is a linear function of $\pi$, that is, for all $M \in \cM$, $f_M(\pi) := w_{M}^T\pi$ for some vector $w_M \in \mathbb{R}^d$. Suppose we have access to an approximate optimization oracle $\mathcal{O}: \mathbb{R}^d \rightarrow \Pi$, which satisfies for some $\gamma < 1$:
\begin{align}
    w^T\mathcal{O}(w) \geq \gamma \max_{\pi \in \Pi}w^T\pi.
\end{align}

Suppose also that the space $\Pi$ has some $\beta$-\em barycentric spanner \em of linearly independent vectors $v_1, \ldots, v_d \in \Pi$ such that for any $\pi \in \Pi$, we can write $\pi = \sum_i c_i v_i$ for $|c_i| \leq \beta$.

\begin{proposition}
For any such $\Pi, \cM$, we have $\mathsf{dec}^{\gamma}(\bM) \leq \eps^{2/3}d$. Further, given query access to $f_{\bM}$, and a $\beta$-barycentric spanner $v_1, \ldots, v_d$, and one call to the oracle $\mathcal{O}$, we can compute the distribution $p = p(\eps, \bM)$ bounding $\mathsf{dec}^{\gamma}_{\eps}(\bM) \leq 3d\beta^{2/3} \eps^{2/3}$. Further, this distribution is continuous in $\eps$.
\end{proposition}
\begin{proof}
For $i \in [d]$, let $y_i := f_{\bM}(v_i)$, and let $w$ be the unique vector for which $w^Tv_i = y_i$ for all $y$. Let $\hat{\pi} := \mathcal{O}(w)$ be the value returned by the oracle, such that 
\begin{align}
    w^T\hat{\pi} \geq \gamma \max_{\pi \in \Pi}w^T\pi.
\end{align}

Let $p$ be the distribution which for each $i$, chooses $v_i$ with probability $\eps^{2/3}\beta^{2/3}$, and chooses $\hat{\pi}$ with probability $1 - d\eps^{2/3}\beta^{2/3}$. Note that this distribution is continuous in $\eps$.

We claim that this $p$ bounds $\mathsf{dec}_{\eps}^{\gamma}(\bM) \leq 3d\beta^{2/3} \eps^{2/3}$. Fix any $M \in \cB_{p, \eps}(\bM)$. Then it must be the case that for all $i \in [d]$, we have
\begin{align}
    \beta^{2/3}\eps^{2/3}(f_{M}(v_i) - f_{\bM}(v_i))^2 \leq \eps^2,
\end{align}
or equivalently,
\begin{align}
    |f_{M}(v_i) - f_{\bM}(v_i)| \leq \eps^{2/3}\beta^{-1/3}.
\end{align}
Let $\hat{\pi} = \sum_{i = 1}^d c_i v_i$, where $|c_i| \leq \beta$. 
Then 
\begin{align}
    f_{M}(\hat{\pi}) &= w_M^T\hat{\pi} = \sum_{i \in [d]}c_i w_M^Tv_i = \sum_{i \in d}c_i f_M(v_i)\\
    &\geq \sum_{i \in d}c_i f_{\bM}(v_i) - |c_i|\eps^{2/3}\beta^{-1/3}\\
    &\geq \sum_{i \in d}c_i f_{\bM}(v_i) - d\beta^{2/3}\eps^{2/3}\\
    &= w^T\hat{\pi} - d\beta^{2/3}\eps^{2/3}\\
    &\geq \gamma w^T\pi_M -  d\beta^{2/3}\eps^{2/3}.
\end{align}
Similarly, with  $\pi_M = \sum_{i = 1}^d a_i v_i$, where $|a_i| \leq \beta$, we have
\begin{align}
    f_M(\pi_M) &= \sum_{i \in d}a_i f_M(v_i)\\
    &\leq \sum_{i \in d}a_i f_{\bM}(v_i) + |a_i|\beta^{-1/3}\eps^{2/3}\\
    &\leq \sum_{i \in [d]} w_{\bM}^Tv_i + d\beta^{2/3}\eps^{2/3}\\
    &= w^T\pi_M + d\beta^{2/3}\eps^{2/3}.
\end{align}
Thus we have 
\begin{align}
    \gamma f_M(\pi_M) - \mathbb{E}_{\pi \sim p}f_M(\pi_M) &= \left(1 - d\beta^{2/3}\eps^{2/3}\right)\left(\gamma f_M(\pi_M) - f_M(\hat{\pi})\right) + \sum_{i = 1}^d\beta^{2/3}\eps^{2/3}\left(\gamma f_M(\pi_M) - \mathbb{E}_{\pi \sim p}f_M(v_i)\right)\\
    &\leq \gamma f_M(\pi_M) - f_M(\hat{\pi}) + d\beta^{2/3}\eps^{2/3}\\
    &\leq \gamma \left( w^T\pi_M + d\beta^{2/3}\eps^{2/3}\right) - \left(\gamma w^T\pi_M -  d\beta^{2/3}\eps^{2/3}\right) + d\beta^{2/3}\eps^{2/3}\\
    &\leq 3d\beta^{2/3} \eps^{2/3}.
\end{align}
\end{proof}

\section{Proof of Theorem~\ref{prop:decc_rho} and Corollary~\ref{corr:comp}}\label{sec:dec_lb}

We restate and prove Theorem~\ref{prop:decc_rho}.
\thmdec*

\begin{proof}[Proof of Theorem~\ref{prop:decc_rho}]
Let $\{p_t\}$ be any algorithm, and for any model $M$, let $\mathbb{P}^M$ and $\mathbb{E}^M$ denote the probability and expectation operators under ground truth $M$, and algorithm $\{p_t\}$. Further define $p_M := \mathbb{E}^M\left[\frac{1}{T}\sum_{t = 1}^T p_t(\cdot| \mathcal{H}^t)\right]$. 

We define $f_M^* := f_M(\pi_M)$ and $g_M(\pi) := f_M^* - f_M(\pi)$. Although it may seem natural to consider the instantaneous $\gamma$-regret $\gamma f_M^* - f_M(\pi)$, we instead work primarily with the traditional instantaneous regret, $g_M(\pi)$, because this value can never be negative. Throughout, one should think of $\eps$ as being small relative to $\rho$ and $\Delta$.

Let $\amax := \Delta + (1 - \gamma)f_{\bM}(\pi_{\bM})$, and for $a \in [0, \amax - \frac{2}{T}]$, consider the following algorithm $\op^a_t$: Play actions according to $p_t$ until some time $\tau_a$ when $\sum_{t = 1}^{\tau_a} g_{\bM}(\pi_t) \geq aT$. Then play $\pi_{\bM}$ for the rest of the rounds. Formally, we have
    \begin{align}
        \op^a_t(\cdot| \mathcal{H}_t) = \begin{cases}
        p_t(\cdot| \mathcal{H}_t) & t \leq \tau_a \\
        \pi_{\bM} & t > \tau_a,
        \end{cases}
    \end{align}
    where $\tau_a$ is defined to be the stopping time which is the first value of $t$ for which $\sum_{s = 1}^{t} g_{\bM}(\pi_s) \geq aT$. If this never occurs, let $\tau := T$.

    For a model $M$, let $\overline{\mathbb{P}}^M_a$ and $\overline{\mathbb{E}}^M_a$ denote the probability and expectation operators under ground truth $M$, and algorithm $\op^a_t$. At a high level, our proof will show that for some value $a$, and some model $M_a$, we have 
    \begin{align}\label{eq:desired}
\overline{\mathbb{P}}_a^{M_a}\left[\frac{1}{T}\sum_{t = 1}^{\tau_a} g_{M_a}(\pi_t) \geq \Delta + (1 - \gamma)f_{M_a}^* - \frac{3\trho}{C} \right] \geq \frac{1}{3C},
\end{align}
where we have defined $\trho := \minterm$.

Then by coupling $\overline{\mathbb{P}}_a^{M_a}$ and $\mathbb{P}^{M_a}$ for the first $\tau_a$ rounds (indeed, the algorithm $p_t$ and $\op^a_t$ both behave the same up to time $\tau_a$), and observing that $g_{M_a}$ is never negative, we can achieve the desired result, that 
    \begin{align}
\mathbb{P}^{M_a}\left[\frac{1}{T}\sum_{t = 1}^{T} \gamma f_{M_a}^* - f_{M_a}(\pi_{t}) \geq \Delta - \frac{3\trho}{C}\right]
&= \mathbb{P}^{M_a}\left[\frac{1}{T}\sum_{t = 1}^{T} g_{M_a}(\pi_t) \geq \Delta + (1 - \gamma)f_{M_a}^* - \frac{3\trho}{C}\right]\\
&\geq \mathbb{P}^{M_a}\left[\frac{1}{T}\sum_{t = 1}^{\tau_a} g_{M_a}(\pi_t) \geq \Delta + (1 - \gamma)f_{M_a}^* - \frac{3\trho}{C}\right] \geq \frac{1}{3C}.
\end{align}

We proceed to prove Equation~\eqref{eq:desired}. Consider the DEC under $p=\op_{\bM}^a = \overline{\mathbb{E}}^{\bM}\left[\frac{1}{T}\sum_{t = 1}^T \op^a_t(\cdot| \mathcal{H}^t)\right]$, and let $M_a$ be the corresponding maximizer in \eqref{eq:def_dec}. Observe that we must have $M_a \neq \bM$ (and thus $M_a \in \cM$), since \begin{align}\label{eq:notbM}
    \mathbb{E}_{\pi \sim \op^a_{\bM}}[g_{\bM}(\pi)] \leq \frac{1}{T}\left(aT + 1\right) < \Delta + (1 - \gamma)f_{\bM}(\pi_{\bM}).
\end{align}
Furthermore, we have
\begin{equation}\label{eq:distance}
    \overline{\mathbb{E}}_a^{\bM}\left[\frac{1}{T}\sum_{t = 1}^{T} (f_{M_a}(\pi_t) - f_{\bM}(\pi_t))^2\right] \leq \eps^2.
    \end{equation}

Our first claim bounds the total variation distance between $\overline{\mathbb{P}}^{M_a}_a$ and $\overline{\mathbb{P}}^{\bM}_a$.
\begin{claim}\label{claim:TV}
\begin{equation}
D_{\text{TV}}\left(\overline{\mathbb{P}}_a^{\bM}, \overline{\mathbb{P}}^{M_a}_a\right) \leq \frac{1}{3C}.
\end{equation}
\end{claim}
\begin{proof}
Using Pinsker's inequality, we can bound the total variation distance by the square root of the KL-divergence. We can then bound the KL-divergence by summing up the KL divergence over each round (see eg. \cite{lattimore2020bandit}, Lemma 15.1). Since we have assumed the noise is Gaussian with variance $1$, the KL divergence at each round is one half times the distance between $f_{M_a}(\pi_t)$ and $f_{\bM}(\pi_t)$ squared. This yields
\begin{align}
D_{\text{TV}}\left(\overline{\mathbb{P}}_a^{\bM}, \overline{\mathbb{P}}^{M_a}_a\right) &\leq \sqrt{D_{\text{KL}}\left(\overline{\mathbb{P}}_a^{\bM}, \overline{\mathbb{P}}^{M_a}_a\right)}\\
&= \sqrt{\overline{\mathbb{E}}_a^{\bM}\left[\sum_{t = 1}^{T} (f_{M_a}(\pi_t) - f_{\bM}(\pi_t))^2\right]}\\
&\leq \eps\sqrt{T} \leq \frac{1}{3C},
\end{align}
where the last inequality follows by assumption of the theorem.
\end{proof}

The next claim bounds the $L_1$ distance between $f_{M_a}$ and $f_{\bM}$ under $\op^a_{\bM}$.
\begin{claim}\label{claim:eps}
For any $a \in [0, \amax)$,
\begin{align}
    \overline{\mathbb{P}}_a^{\bM}\left[\frac{1}{T}\sum_{t = 1}^{\tau_a} |f_{M_a}(\pi_t) - f_{\bM}(\pi_t)| \leq 3C\eps \right] \geq 1 - \frac{1}{3C}.
\end{align}
\end{claim}
\begin{proof}
Applying Jensen's inequality to Equation~\eqref{eq:distance} yields that 
\begin{align}\label{eq:eps}
    \overline{\mathbb{E}}_a^{\bM}\left[\frac{1}{T}\sum_{t = 1}^{T} |f_{M_a}(\pi_t) - f_{\bM}(\pi_t)|\right] \leq \eps.
\end{align}

Thus by Markov's we have 
\begin{align}
    \overline{\mathbb{P}}_a^{\bM}\left[\frac{1}{T}\sum_{t = 1}^{T} f_{M_a}(\pi_t) - f_{\bM}(\pi_t) \geq 3C\eps \right] \leq \frac{1}{3C}.
\end{align}
The conclusion follows.
\end{proof}

Our next two claims show that there is a large gap between the maximum of $\bM$ and $\gamma$ times the maxima of other models in $\cM$. We will need to leverage both of these claims separately.

\begin{claim}\label{claim:diff}
For any $a \in [0, \amax]$,
\begin{align}
   \gamma f_{M_a}^* - f_{\bM}^* \geq \Delta - a - 2\eps.
\end{align}
\end{claim}
Note that the right-hand side in this claim may be negative for large values of $a$; however, we will only use this claim when $a < \Delta - 2\eps$.
\begin{proof}
Suppose not, that is, $f_{M_a}^* - f_{\bM}^* < (1 - \gamma) f_{M_a}^* + \Delta - a - 2\eps$. Then
\begin{align}
\overline{\mathbb{E}}_a^{\bM}\sum_{t = 1}^T g_{M_a}(\pi_t) &= \overline{\mathbb{E}}_a^{\bM}\sum_{t = 1}^{T} \left(g_{\bM}(\pi_t) + f_{M_a}^* - f_{\bM}^* + f_{\bM}(\pi_t) - f_{M_a}(\pi_t)\right)\\
&<\overline{\mathbb{E}}_a^{\bM}\sum_{t = 1}^{T} \left(g_{\bM}(\pi_t)\right) +\left((1 - \gamma) f_{M_a}^* + \Delta - a - 2\eps\right)T + \eps T\\
&= \overline{\mathbb{E}}_a^{\bM}\sum_{t = 1}^{\tau_a} \left(g_{\bM}(\pi_t)\right) + \left((1 - \gamma) f_{M_a}^* + \Delta - a - 2\eps\right)T + \eps T\\
&\leq  \left(aT + 1\right) + \left((1 - \gamma) f_{M_a}^* + \Delta - a - 2\eps\right)T + \eps T\\
&\leq \left((1 - \gamma) f_{M_a}^* + \Delta\right) T,
\end{align}
Here the first inequality uses Equation~\eqref{eq:eps}. The second equality uses the fact that $\op^a_t$ plays $\pi_{\bM}$ after time $\tau_a$, and $g_{\bM}(\pi_{\bM}) = 0$. The third inequality uses the definition of $\tau_a$.
The final result is a contradiction, since $M_a$ is the maximizer of DEC for $p=\op_{\bM}^a$, and thus we must have 
\begin{align}
    \overline{\mathbb{E}}_a^{\bM}\sum_{t = 1}^T g_{M_a}(\pi_t) &\geq \left((1 - \gamma) f_{M_a}^* + \Delta \right)T.
\end{align}
\end{proof}

\begin{claim}\label{claim:gap_full}
    For any $M \in \cM$, we have $\gamma f_{M}^* - f_{\bM}^* \geq \Delta - \gamma \rho - 2\eps.$
\end{claim}
\begin{proof}
    Applying Claim~\ref{claim:diff} with with $a = 0$, we observe that for some $M_0 \in \cM$, 
    \begin{align}
        \gamma f_{M_0}^* - f_{\bM}^* \geq \Delta - 2\eps.
    \end{align}
    The claim now follows from the fact that $\cM$ is $\rho$-localized, so $f_M^* \geq 1 - \rho \geq f_{M_0}^* - \rho$ for any $M \in \cM$.
\end{proof}

The following two claims are the crux of the proof. In them, we relate the probability that certain $M_a$ will achieve a large regret by time $\tau_a$ to the probability under $\mathbb{P}^{\bM}$ that $\frac{1}{T}\sum_{t = 1}^T g_{\bM}(\pi_t)$ falls in a certain interval.

\begin{claim}\label{claim:crux3}
Fix any values $a'$ and $a$ in $[0, \amax]$, with $a' < a$. Then
\begin{align}
\overline{\mathbb{P}}_a^{\bM}\left[\frac{1}{T}\sum_{t = 1}^{\tau_a} g_{M_a}(\pi_t) \geq \Delta + (1 - \gamma)f_{M_a}^* - a + a' - (2 + 3C)\eps \right] \geq  \mathbb{P}^{\bM}\left[a' \leq \frac{1}{T}\sum_{t = 1}^T g_{\bM}(\pi_t) < a\right] - \frac{1}{3C}.
\end{align}

\end{claim}

\begin{claim}\label{claim:crux1}
Let $a = \amax - \frac{2}{T} = \Delta + (1 - \gamma)f_{\bM}(\pi_{\bM}) - \frac{2}{T}$. Then for any $a' \leq \trho$,
\begin{align}
\overline{\mathbb{P}}_a^{\bM}\left[\frac{1}{T} \sum_{t = 1}^{\tau_a} g_{M_a}(\pi_t) \geq \Delta + (1 - \gamma)f_{M_a}^* - \trho + a' - (4 + 3C)\eps \right] \geq \mathbb{P}^{\bM}\left[\frac{1}{T}\sum_{t = 1}^T g_{\bM}(\pi_t) \geq a'\right] - \frac{1}{3C}.
\end{align}
\end{claim}

Claim~\ref{claim:crux3} achieves a meaningful result whenever $a - a'$ is small relative to $\Delta$. Claim~\ref{claim:crux1} achieves a meaningful result whenever $a'$ is close to $\trho$.

\begin{proof}[Proof of Claim~\ref{claim:crux3}]
Recall from Claim~\ref{claim:diff} that we have $\gamma f_{M_a}^* \geq f_{\bM}^* + \Delta - 2\eps - a.$ Thus if $\tau_a = T$, we have 
    \begin{align}
        \sum_{t = 1}^{\tau_a} g_{M_a}(\pi_t) &= \sum_{t = 1}^{T} f_{M_a}^* - f_{M_a}(\pi_t)  \\
        &\geq \sum_{t = 1}^{T} (f_{M_a}^* - f_{\bM}^*) + g_{\bM}(\pi_t) - |f_{\bM}(\pi_t) - f_{M_a}(\pi_t)|\\
        &\geq \left(\Delta + (1 - \gamma)f_{M_a}^* - 2\eps - a\right)T + \sum_{t = 1}^T  g_{\bM}(\pi_t) - |f_{\bM}(\pi_t) - f_{M_a}(\pi_t)|.
    \end{align}
If $a'T \leq \sum_{t = 1}^{T} g_{\bM}(\pi_t) < a T$ then $\tau_a = T$, and the value above is at least 
\begin{align}
     \left(\Delta + (1 - \gamma)f_{M_a}^* - 2\eps - a + a'\right)T - \sum_{t = 1}^T |f_{\bM}(\pi_t) - f_{M_a}(\pi_t)|.
\end{align}

Thus by a union bound, we have 
\begin{align} \overline{\mathbb{P}}_a^{\bM}\left[\frac{1}{T} \sum_{t = 1}^{\tau_a} g_{M_a}(\pi_t) \geq \Delta + (1 - \gamma)f_{M_a}^* - a + a' - (2 + 3C)\eps \right] &\geq \overline{\mathbb{P}}_a^{\bM}\left[a'T \leq \sum_{t = 1}^{\tau_a} g_{\bM}(\pi_t) < a T\right]\\
&\qquad - \overline{\mathbb{P}}_a^{\bM}\left[\sum_{t = 1}^{\tau_a} |f_{\bM}(\pi_t) - f_{M_a}(\pi_t)| \geq 3C\eps T\right].
\end{align}
By Claim~\ref{claim:eps}, the second term on the right hand side is at most $\frac{1}{3C}$.

Finally, observe that 
\begin{align} \overline{\mathbb{P}}_a^{\bM}\left[a'T \leq \sum_{t = 1}^{\tau_a} g_{\bM}(\pi_t) < a T\right] 
&= \mathbb{P}^{\bM}\left[a'T \leq \sum_{t = 1}^{\tau_a} g_{\bM}(\pi_t) < a T\right]\\
&= \mathbb{P}^{\bM}\left[a'T \leq \sum_{t = 1}^{T} g_{\bM}(\pi_t) < a T\right].
\end{align}
where the first step follows because of the coupling between $\overline{\mathbb{P}}_a^{\bM}$ and $\mathbb{P}^{\bM}$ for events up to time $\tau_a$, and the second step follows because if $\sum_{t = 1}^{T} g_{\bM}(\pi_t) < a T$, then $\tau = T$, and vice versa.

This proves the claim.
\end{proof}

The proof of Claim~\ref{claim:crux1} is very similar, though we need to leverage the gap from Claim~\ref{claim:gap_full} in addition to Claim~\ref{claim:diff}.
\begin{proof}[Proof of Claim~\ref{claim:crux1}]
Observe that
    \begin{align}\label{eq:combo}
        \sum_{t = 1}^{\tau_a} g_{M_a}(\pi_t) &= \sum_{t = 1}^{\tau_a} f_{M_a}^* - f_{M_a}(\pi_t)  \\
        &\geq \sum_{t = 1}^{\tau_a} (f_{M_a}^* - f_{\bM}^*) + g_{\bM}(\pi_t) - |f_{\bM}(\pi_t) - f_{M_a}(\pi_t)|\\
        &\geq \left((1 - \gamma)f_{M_a}^* + \Delta - \trho - 2\eps)\right)\tau_a + \sum_{t = 1}^{\tau_a}g_{\bM}(\pi_t) - |f_{\bM}(\pi_t) - f_{M_a}(\pi_t)|,
    \end{align}
where the last line follows from combining Claims~\ref{claim:diff} and Claim~\ref{claim:gap_full} to yield 
\begin{align}
    \gamma f_{M_a}^* - f_{\bM}^* \geq \Delta - \min(a, \gamma \rho) - 2\eps \geq \Delta - \min(\Delta + (1 - \gamma)f_{\bM}^*, \gamma \rho) - 2\eps \geq \Delta - \trho - 2\eps.
\end{align}

If $a'T \leq \sum_{t = 1}^{T} g_{\bM}(\pi_t) < a T$, then $\tau_a = T$, and thus we have 
\begin{align}
    \sum_{t = 1}^{\tau_a} g_{M_a}(\pi_t) &\geq \left((1 - \gamma)f_{M_a}^* + \Delta -  \trho + a' - 2\eps\right)T - \sum_{t = 1}^T |f_{\bM}(\pi_t) - f_{M_a}(\pi_t)|.
\end{align}
Alternatively, if $\sum_{t = 1}^{T} g_{\bM}(\pi_t) \geq a T = (\Delta + (1 - \gamma)f_{\bM}^*)T - 2$, then from the second line of Equation~\eqref{eq:combo}, we obtain:
\begin{align}\label{eq:taua}
    \sum_{t = 1}^{\tau_a} g_{M_a}(\pi_t) &\geq \left(f_{M_a}^* - f_{\bM}^*\right)\tau_a + \sum_{t = 1}^{\tau_a} g_{\bM}(\pi_t) - \sum_{t = 1}^T |f_{\bM}(\pi_t) - f_{M_a}(\pi_t)|\\
    &\geq \left(f_{M_a}^* - f_{\bM}^*\right)\tau_a + (\Delta + (1 - \gamma)f_{\bM}^*)T - 2 - \sum_{t = 1}^T |f_{\bM}(\pi_t) - f_{M_a}(\pi_t)|\\ 
    &\geq \left(f_{M_a}^* - f_{\bM}^*\right)\tau_a + (\Delta + (1 - \gamma)f_{M_a}^* - 2\eps)T - (1 - \gamma)(f_{M_a}^*- f_{\bM}^*)T\\
    &\qquad - \sum_{t = 1}^T |f_{\bM}(\pi_t) - f_{M_a}(\pi_t)|.
\end{align}
Now we lower bound $\tau_a$. Since $g_{\bM}(\pi) \leq f_{\bM}^*$ for any $\pi$, we have 
\begin{align}
    \tau_a f_{\bM}^* &\geq \sum_{t = 1}^{\tau_a} g_{\bM}(\pi_t) \\
    &\geq \left(\Delta + (1 - \gamma)f_{\bM}^*\right)T - 2 \\
    &\geq (1 - \gamma)f_{\bM}^*T,
\end{align}
where in the second inequality we used the definition of $a$, and in the third inequality we used the assumption of the theorem that $\Delta \geq \eps \geq 2/T$. Thus we have $\tau_a \geq (1 - \gamma)T$. 

Observe also that by Claim~\ref{claim:diff}, and the definition of $a$, we have $\gamma f_{M_a}^* - f_{\bM}^* \geq \Delta - (\Delta + (1 - \gamma)f_{\bM}^*) - 2\eps$, and thus $f_{M_a}^* - f_{\bM}^*  \geq -2\eps/\gamma$.

Plugging this and the lower bound on $\tau_a$ into Equation~\eqref{eq:taua} yields that if $\sum_{t = 1}^{T} g_{\bM}(\pi_t) \geq aT$, then
\begin{align}
     \sum_{t = 1}^{\tau_a} g_{M_a}(\pi_t) &\geq \left(f_{M_a}^* - f_{\bM}^*\right)(\tau_a - (1 - \gamma) T) + (\Delta + (1 - \gamma)f_{M_a}^* - 2\eps)T - \sum_{t = 1}^T |f_{\bM}(\pi_t) - f_{M_a}(\pi_t)| \\
     &\geq -2\eps T  + (\Delta + (1 - \gamma)f_{M_a}^* - 2\eps)T - \sum_{t = 1}^T |f_{\bM}(\pi_t) - f_{M_a}(\pi_t)| \\
     &\geq (\Delta + (1 - \gamma)f_{M_a}^* - 4\eps)T - \sum_{t = 1}^T |f_{\bM}(\pi_t) - f_{M_a}(\pi_t)|\\
     &\geq (\Delta + (1 - \gamma)f_{M_a}^* - 4\eps)T - \trho T + a'T - \sum_{t = 1}^T |f_{\bM}(\pi_t) - f_{M_a}(\pi_t)|.
\end{align}
Here in the second inequality, we used the fact that $\tau_a \geq (1 - \gamma)T$ and $f_{M_a}^* - f_{\bM}^* \geq -2\eps/\gamma$. Thus if $f_{M_a}^* - f_{\bM}^* \geq 0$, the step follows from the fact that $\tau_a \geq (1 - \gamma)T$. If $f_{M_a}^* - f_{\bM}^* \in [-2\eps/\gamma, 0)$, then the step follows from the fact that $\tau_a - (1 - \gamma)T \leq \gamma T$. In the final inequality, we used the assumption of the claim that $a' \leq \trho$. 

Thus by a union bound, we have
\begin{align} \overline{\mathbb{P}}_a^{\bM}\left[\frac{1}{T} \sum_{t = 1}^{\tau_a} g_{M_a}(\pi_t) \geq \Delta + (1 - \gamma)f_{M_a}^* - \trho + a' - (4 + 3C)\eps \right] &\geq \overline{\mathbb{P}}_a^{\bM}\left[a'T \leq \sum_{t = 1}^{\tau_a} g_{\bM}(\pi_t) < a T\right]\\
&\qquad + \overline{\mathbb{P}}_a^{\bM}\left[ \sum_{t = 1}^{\tau_a} g_{\bM}(\pi_t) \geq a T\right]\\
&\qquad -\overline{\mathbb{P}}_a^{\bM}\left[\sum_{t = 1}^{\tau_a} |f_{\bM}(\pi_t) - f_{M_a}(\pi_t)| \geq 3C\eps T\right].
\end{align}
By Claim~\ref{claim:eps}, the second term on the right hand side is at most $\frac{1}{3C}$. Finally, observe that 
\begin{align} \overline{\mathbb{P}}_a^{\bM}\left[a'T \leq \sum_{t = 1}^{\tau_a} g_{\bM}(\pi_t) < a T\right] + \overline{\mathbb{P}}_a^{\bM}\left[\sum_{t = 1}^{\tau_a} g_{\bM}(\pi_t) \geq a T\right] &= \overline{\mathbb{P}}_a^{\bM}\left[\sum_{t = 1}^{\tau_a} g_{\bM}(\pi_t) \geq a' T\right] \\
&= \mathbb{P}^{\bM}\left[\sum_{t = 1}^{\tau_a} g_{\bM}(\pi_t) \geq a' T\right]\\
&\geq \mathbb{P}^{\bM}\left[\sum_{t = 1}^{T} g_{\bM}(\pi_t) \geq a' T\right].
\end{align}
where the second step follows because of the coupling between $\overline{\mathbb{P}}_a^{\bM}$ and $\mathbb{P}^{\bM}$ for events up to time $\tau_a$, and the second to last step follows because if $\sum_{t = 1}^{T} g_{\bM}(\pi_t) \geq a' T$, then $\sum_{t = 1}^{\tau_a} g_{\bM}(\pi_t) \geq a' T$. 

This proves the claim.
\end{proof}

Now we show how to instantiate Claims~\ref{claim:crux3} and \ref{claim:crux1} to prove the theorem. Let $C$ be the integer from the theorem statement. For $i = 1, 2, \ldots, C - 1$, instantiate Claim~\ref{claim:crux3} with $a' = \trho \frac{i-1}{C}$, $a = \trho \frac{i}{C}$. 

Then, instantiate Claim~\ref{claim:crux1} with $a' = \trho \frac{C - 1}{C}$. 

Now clearly for at least one of the $C$ intervals $I_i := \left[\trho\frac{i-1}{C}, \gamma \frac{i}{C}\right)$ for $i \in [C - 1]$ and the final interval $I_C := \left[\trho \frac{C - 1}{C}, 1\right]$, we have 
\begin{equation}
    \mathbb{P}^{\bM}\left[\frac{1}{T}\sum_{t = 1}^Tg_{\bM}(\pi_t) \in I_i\right] \geq \frac{1}{C}.
\end{equation}
Then by choosing $a = \trho \frac{i}{C}$ if $i \leq C - 1$, and otherwise $a = \Delta + (1 - \gamma)f_{\bM}^* - \frac{2}{T}$ if $i = C$, we have
\begin{align}
\overline{\mathbb{P}}_a^{\bM}\left[\frac{1}{T}\sum_{t = 1}^{\tau_a} g_{M_a}(\pi_t) \geq \Delta + (1 - \gamma)f_{M_a}^* - \frac{\trho}{C} - (4 + 3C)\eps \right] \geq \frac{1}{C} - \frac{1}{3C}.
\end{align}
Indeed, this works for $i \leq C - 1$ since we instantiate Claim~\ref{claim:crux3} with $a - a' = \frac{\trho}{C}$. If $i = C$, as mentioned above, we use $a' =\trho \frac{C - 1}{C}$, and thus Claim~\ref{claim:crux1} yields the result above. 

Applying the change of measure to $\overline{\mathbb{P}}^{M_a}$, by Claim~\ref{claim:TV}, for this value of $a$, we have 

\begin{align}
\overline{\mathbb{P}}_a^{M_a}\left[\frac{1}{T}\sum_{t = 1}^{\tau_a} g_{M_a}(\pi_t) \geq \Delta + (1 - \gamma)f_{M_a}^* - \frac{\trho}{C} - (4 + 3C)\eps \right] \geq \frac{1}{C} - \frac{1}{3C} - \frac{1}{3C} = \frac{1}{3C}.
\end{align}
Observing that by the assumption of the theorem that $(4 + 3C)\eps \leq \frac{2\trho}{C}$, this proves Equation~\eqref{eq:desired}, as desired.

\end{proof}

We restate and prove Corollary~\ref{corr:comp}.
\compp*
\begin{proof}[Proof of Corollary~\ref{corr:comp}]
Suppose there existed an algorithm $\{p_t\}$ which ran in time $R/\log(200C^2/\delta)$ and achieved
\begin{align}
   \mathbb{P}^M\left[\frac{1}{T}\mathsf{Reg}_{\gamma}(T)\geq \Delta - \frac{3\min(\rho, \Delta + 1 - \gamma)}{C}\right] \leq \frac{1}{6C}.
\end{align}
We will show how to use this algorithm to construct an algorithm $A$ which upper bounds $\mathsf{dec}^{\gamma}_{\eps} \leq \Delta$ in $R$ time. Fix any reference model $\bM$ for which we want to find some $p = A(\bM)$ which upper bounds the DEC by $\Delta$.

Following the proof of Theorem~\ref{prop:decc_rho} with a slight modification, define $\op^a_t$ to be the algorithm which plays actions according to $p_t$ until some time $\tau_a$ when $\sum_{t = 1}^{\tau_a} g_{\bM}(\pi_t) \geq aT$. Then play $\hat{\pi}$ for the rest of the rounds, where $\hat{\pi}$ is such that $f_{\bM}(\hat{\pi}) \geq f_{\bM}(\pi_{\bM}) - \frac{\Delta}{10}$. Such a $\hat{\pi}$ can be found in time $R$ by Assumption~\ref{assm:eo}. Formally, we have
    \begin{align}
        \op^a_t(\cdot| \mathcal{H}_t) = \begin{cases}
        p_t(\cdot| \mathcal{H}_t) & t \leq \tau_a \\
        \hat{\pi} & t > \tau_a,
        \end{cases}
    \end{align}
    where $\tau_a$ is defined to be the stopping time which is the first value of $t$ for which $\sum_{s = 1}^{t} g_{\bM}(\pi_s) \geq aT$.
Since we have changed the strategy $\op_t^a$ from the proof of Theorem~\ref{prop:decc_rho}, by playing $\hat{\pi}$ instead of $\pi_{\bM}$, we need to check that the model $M_a$ maximizing the DEC under the distribution $\op_{\bM}^a$ is not $\bM$. Reproducing Equation~\eqref{eq:notbM}, we have
\begin{align}
     \mathbb{E}_{\pi \sim \op^a_{\bM}}[g_{\bM}(\pi)] &\leq \frac{1}{T}\left(aT + 1 + T(f_{\bM}(\pi_{\bM}) - f_{\bM}(\hat{\pi}))\right)\\
     &< \Delta + (1 - \gamma)f_{\bM}(\pi_{\bM}) + \frac{\Delta}{10}\\
     &= \frac{11}{10}\Delta + (1 - \gamma)f_{\bM}(\pi_{\bM}),
\end{align}
and thus $\bM$ cannot attain a $\gamma$-regret of at least $\frac{11}{10}\Delta$ under this distribution. Thus $M_a \in \cM$.

Note that for any given $a$, we can output a sample from each of $\{\op_t^a\}_{t = 1}^T$ in $\frac{R}{50C^2\log(C/\delta)} + R$ time, assuming query access to $f_{\bM}$. This is true because at each step $t$, we just need to call $p_t$ and also compute $g_{\bM}(\pi_{t - 1}) = f_{\bM}(\pi_{\bM}) - f_{\bM}(\pi_{t-1})$ to check whether we have reached the stopping time. Ultimately, we will want to output samples from the distribution $\op^a := \frac{1}{T}\sum_t\op^a_t$, but first we will need the algorithm $A$ to identify the correct choice of $a$.

Consider the $C$ intervals $I_i$ given in the proof of Theorem~\ref{prop:decc_rho}. We must have for at least one of these intervals $i$ that 
\begin{align}
    \mathbb{P}^{\bM}\left[\frac{1}{T}\sum_{t = 1}^{T} g_{\bM}(\pi_t) \in I_i\right] \geq \frac{1}{C},
\end{align}
where here the probability is over running $\{p_t\}$ in the environment $\bM$. Since we can simulate play in the environment $\bM$ via query access to $f_{\bM}$, our algorithm $A$ will first find a $i$ for which with probability $1 - \delta/2$, it holds that
\begin{align}
    \mathbb{P}^{\bM}\left[\frac{1}{T}\sum_{t = 1}^T g_{\bM}(\pi_t) \in I_i\right] \geq \frac{9}{10C}.
\end{align}

For $j = 1, \ldots, J$, let $X_i^j$ denote the indicator of the event that when we run the bandit algorithm $\{p_t\}$ for $t = 1$ to $T$, we have $\sum_{t = 1}^T g_{\bM}(\pi_t) \in I_i$. Then for $J = 50C^2\log(C/\delta)$, for each $i$, by Hoeffding's inequality, with probability $1 - \delta/C$, we have 
\begin{align}
    \left|\mathbb{P}^{\bM}\left[\frac{1}{T}\sum_{t = 1}^T g_{\bM}(\pi_t) \in I_i\right] - \frac{1}{J}\sum_j X_{i}^j\right| \leq \frac{1}{10C}.
\end{align}
Thus with probability $1 - \delta$, in time $J \frac{R}{50C^2\log(C/\delta)} + R = 2R$, we can identify some $i$ for which 
\begin{align}
    \mathbb{P}^{\bM}\left[\frac{1}{T}\sum_{t = 1}^T g_{\bM}(\pi_t) \in I_i\right] \geq \frac{9}{10C}.
\end{align}

The proof of Theorem~\ref{prop:decc_rho} now guarantees the following for this $i$: Let $a$ be the choice of $a$ associated with the interval $I_i$, that is, for $i \in [C - 1]$, we have $a = \trho \frac{i}{C}$, and for $i = C$, we have $a = \amax - \frac{2}{T} := \Delta + (1 - \gamma)f_{\bM}(\pi_{\bM}) - \frac{2}{T}$). Then it must be the case that $\op^a$ gives an upper bound of $\Delta$ on the $\mathsf{dec}^{\gamma}_{\eps}$; otherwise, it would be the case that there exists some $M_a \in \cM$ such that 
\begin{align}
\overline{\mathbb{P}}_a^{M_a}\left[\frac{1}{T}\sum_{t = 1}^{\tau_a} g_{M_a}(\pi_t) \geq \Delta + (1 - \gamma)f_{M_a}^* - \frac{\trho}{C} - (4 + 3C)\eps \right] \geq \frac{9}{10C} - \frac{1}{3C} - \frac{1}{3C} \geq \frac{1}{6C},
\end{align}
and thus by the logic the proof of Theorem~\ref{prop:decc_rho}, with probability $1 - \frac{1}{6C}$, the regret under $M_a$ is at least $\Delta - \frac{3\min(\rho, \Delta + 1 - \gamma)}{C}$.

This contradicts the hypothesis of our corollary, and thus it must be the case that $\op^a$ gives an upper bound of $\Delta$ on the $\mathsf{dec}^{\gamma}_{\eps}$. This proves the corollary. 
\end{proof}

\section{Proof of Theorems~\ref{thm:simpleub} and \ref{thm:ub}}\label{sec:ubproof}
\subsection{Proof of Theorem~\ref{thm:simpleub}}
We restate and prove Theorem~\ref{thm:simpleub}.
\simpub*
\begin{proof}[Proof of Theorem~\ref{thm:simpleub}]
Condition on the event that 
\begin{align}\label{eq:event}\sum_{t = 1}^{T} \mathbb{E}_{\pi \sim p_t}(f_{M^*}(\pi) - f_{\bM_t}(\pi))^2 \leq \mathsf{Est}(T, \delta),
\end{align}
which occurs with probability $1 - \delta$.

Our first goal will be to show that for any $t$, the gap $\gamma f_{M^*}(\pi_{M^*}) - f_{\bM_t}(\pi_{\bM_t})$ is small.
\begin{claim}\label{claim:gapUB}
For any $t$, we have 
\begin{align}
    \gamma f_{M^*}(\pi_{M^*}) - f_{\bM_t}(\pi_{\bM_t}) \leq \mathsf{dec}_{\eps_{\lceil{t/2}\lceil}}^{\gamma} + \frac{4}{\sqrt{t}}\sqrt{\mathsf{Est}(T, \delta)}.
\end{align}
\end{claim}
\begin{proof}
Fix $t$, and define
\begin{align}
   S := \{t/2 \leq  s < t: \mathbb{E}_{\pi \sim p_s}(f_{M^*}(\pi) - f_{\bM_s}(\pi))^2 \leq \eps_s^2 \}\}, 
\end{align} which is the set of rounds between $t/2$ and $t$ under which the estimator is close to $M^*$ on the queried distribution. 
We can  bound $|S| \geq \lceil{t/4}\rceil$ using Equation~\eqref{eq:event} as follows. If $|S| < \lceil{t/4}\rceil$, then then for at least $\lfloor{t - t/2}\rfloor - |S|$ rounds $s$ between $t/2$ and $t$, we have $(f_{M^*}(\pi) - f_{\bM_s}(\pi))^2 > \eps_s^2 \geq \eps_t^2$. Thus we would have 
\begin{align}\sum_{s = \lceil{t/2}\rceil}^{t -1} \mathbb{E}_{\pi \sim p_s}(f_{M^*}(\pi) - f_{\bM_s}(\pi))^2 > \left(\lfloor{t - t/2}\rfloor - (\lceil{t/4}\rceil - 1)\right)\eps^2_{t} \geq \max\left(1, \lfloor{t/6\rfloor}\right) \eps^2_{t} \geq \mathsf{Est}(T, \delta),
\end{align}
which contradicts Equation~\eqref{eq:event}.

For the sake of analysis, construct the distribution
$$\hat{p} := \frac{1}{|S|}\sum_{s \in S}p_s.$$

By Assumption~\ref{assm:oracle}, we can write $\bM_t$ as a convex combination of models in $\cM_{t-1}$, which we recall is defined in Assumption~\ref{assm:oracle} as $\{M \in \cM: \sum_{s = 1}^{t} \mathbb{E}_{p_s}[(f_{M^*}(\pi_s)- f_{\bM_s}(\pi_s))^2] \leq \mathsf{Est}(T, \delta)\}$. Let $\nu$ denote the distribution of this convex combination, such that $\bM_t = \mathbb{E}_{M \sim \nu} M$. Then we have
\begin{align}\label{eq:zero}
\gamma f_{M^*}(\pi_{M^*})& - f_{\widehat{M}_t}(\pi_{\widehat{M}_t})\\
&\leq \mathbb{E}_{\pi \sim \hat{p}}\left[\gamma f_{M^*}(\pi_{M^*}) - f_{\widehat{M}_t}(\pi)\right] \\
&= \mathbb{E}_{M \sim \nu}\mathbb{E}_{\pi \sim \hat{p}}\left[\gamma f_{M^*}(\pi_{M^*}) - f_{M}(\pi)\right]\\
&\leq \mathbb{E}_{M \sim \nu}\mathbb{E}_{\pi \sim \hat{p}}\left[\gamma f_{M^*}(\pi_{M^*}) - f_{M^*}(\pi)\right] + \mathbb{E}_{M \sim \nu}\mathbb{E}_{\pi \sim \hat{p}}|f_M(\pi) - f_{M^*}(\pi)|\\
&= \mathbb{E}_{\pi \sim \hat{p}}\left[\gamma f_{M^*}(\pi_{M^*}) - f_{M^*}(\pi)\right] + \mathbb{E}_{M \sim \nu}\mathbb{E}_{\pi \sim \hat{p}}|f_M(\pi) - f_{M^*}(\pi)|.
\end{align}

Now for each $s \in S$, since by definition of $S$, we have $\mathbb{E}_{\pi \sim p_s}(f_{M^*}(\pi) - f_{\bM_s}(\pi))^2 \leq \eps_s^2$. Thus since $p_s$ is chosen to minimize the DEC, we have 
\begin{align}
    \mathbb{E}_{\pi \sim p_s}\left[\gamma f_{M^*}(\pi_{M^*}) - f_{M^*}(\pi)\right] \leq \mathsf{dec}_{\epsilon_s}^{\gamma}.
\end{align}

Thus
\begin{align}\label{eq:firstterm}
    \mathbb{E}_{\pi \sim \hat{p}}\left[\gamma f_{M^*}(\pi_{M^*}) - f_{M^*}(\pi)\right] = \frac{1}{|S|}\sum_{s \in S}\mathbb{E}_{\pi \sim p_s}\left[\gamma f_{M^*}(\pi_{M^*}) - f_{M^*}(\pi)\right] \leq  \frac{1}{|S|}\sum_{s \in S}\mathsf{dec}_{\epsilon_s}^{\gamma} \leq \mathsf{dec}_{\epsilon_{\lceil{t/2}\rceil}}^{\gamma},
\end{align}
where the final inequality holds because all $s \in S$ are at least $t/2$.

We now consider the second term $\mathbb{E}_{M \sim \nu}\mathbb{E}_{\pi \sim \hat{p}}|f_M(\pi) - f_{M^*}(\pi)|$. For any $M \in \cM_{t -1}$, we have
\begin{align}\label{eq:secondterm}
\mathbb{E}_{\pi \sim \hat{p}}|f_M(\pi) - f_{M^*}(\pi)| &\leq \sqrt{\mathbb{E}_{\hat{p}}[(f_{M^*}(\pi)- f_{M}(\pi))^2]}\\
&= \sqrt{\frac{1}{|S|}\sum_{s \in S}\mathbb{E}_{p_s}[(f_{M^*}(\pi)- f_{M}(\pi))^2]} \\
&\leq  \sqrt{\frac{1}{|S|}\sum_{s \in S}\mathbb{E}_{p_s}[(f_{M^*}(\pi)- f_{\bM_s}(\pi))^2]} + \sqrt{\frac{1}{|S|}\sum_{s \in S}\mathbb{E}_{p_s}[(f_{\bM_s}(\pi)- f_{M}(\pi))^2]}\\
&\leq \sqrt{\frac{4}{t}\sum_{s < t}\mathbb{E}_{p_s}[(f_{M^*}(\pi)- f_{\bM_s}(\pi))^2]} + \sqrt{\frac{4}{t}\sum_{s < t}\mathbb{E}_{p_s}[(f_{\bM_s}(\pi)- f_{M}(\pi))^2]}\\
&\leq \frac{4}{\sqrt{t}}\sqrt{\mathsf{Est}(T, \delta)},
\end{align}
Here the first inequality follows from Jenson's inequality, the second inequality follows from triangle inequality and the third from the fact that $|S| \geq t/4$. The final inequality follows from the fact that both $M^*$ and $M$ are in $\cM_{t-1}$, and by definition of $\cM_{t-1}$, for any $\tM \in \cM_{t - 1}$, we have $\sum_{s < t}\mathbb{E}_{p_s}[(f_{M^*}(\pi)- f_{\tM}(\pi))^2] \leq \mathsf{Est}(T, \delta)$.

Plugging Equations~\ref{eq:firstterm} and \ref{eq:secondterm} into Equation~\ref{eq:zero}, we achieve the claim:
\begin{align}
    \gamma f_{M^*}(\pi_{M^*})& - f_{\widehat{M}_t} \leq \mathsf{dec}_{\epsilon_s}^{\gamma} + \frac{4}{\sqrt{t}}\sqrt{\mathsf{Est}(T, \delta)}
\end{align}
\end{proof}

We are now ready to analyze the total $\gamma$-regret. We have
\begin{align}\label{eq:3part}
    \sum_{t = 1}^T \gamma f_{M^*}(\pi_{M^*})& - \mathbb{E}_{p_t}f_{M^*}(\pi)\\
    &\leq \sum_{t = 1}^T  \gamma f_{M^*}(\pi_{M^*}) - \mathbb{E}_{p_t}f_{\bM_t}(\pi) + \sum_{t = 1}^T  \mathbb{E}_{p_t}|f_{\bM_t}(\pi) - f_{M^*}(\pi)|\\
    &= \sum_{t = 1}^T \gamma f_{M^*}(\pi_{M^*}) - f_{\bM_t}(\pi_{\bM_t}) + \sum_{t = 1}^T f_{\bM_t}(\pi_{\bM_t}) - \mathbb{E}_{p_t}f_{\bM_t}(\pi) + \sum_{t = 1}^T  \mathbb{E}_{p_t}|f_{\bM_t}(\pi) - f_{M^*}(\pi)|\\
\end{align}
For the first term, by Claim~\ref{claim:gapUB}, we have,
\begin{align}
    \sum_{t = 1}^T \gamma f_{M^*}(\pi_{M^*}) - f_{\bM_t}(\pi_{\bM_t}) \leq \sum_{t = 1}^T \left(\mathsf{dec}_{\eps_{\lceil{t/2}\rceil}}^{\gamma} + \frac{4}{\sqrt{t}}\sqrt{\mathsf{Est}(T, \delta)}\right) \leq  \left(\sum_{t = 1}^T \mathsf{dec}_{\eps_{\lceil{t/2}\rceil}}^{\gamma}\right) + 8\sqrt{T\mathsf{Est}(T, \delta)}
\end{align}
For the second term in Equation~\eqref{eq:3part}, by definition of the fact that the distribution $p_t$ minimizes the $\gamma$-DEC, we have
\begin{align}
    \sum_{t = 1}^T f_{\bM_t}(\pi_{\bM_t}) - \mathbb{E}_{p_t}f_{\bM_t}(\pi) \leq \sum_{t = 1}^T (1 - \gamma)f_{\bM_t}(\pi_{\bM_t}) + 
\mathsf{dec}_{\eps_{t}}^{\gamma}.
\end{align}
Finally for the third term in Equation~\eqref{eq:3part}, we have by Jensen's inequality (applied to the distribution $\frac{1}{T}\sum_{t = 1}^Tp_t$) that 
\begin{align}
    \sum_{t = 1}^T  \mathbb{E}_{p_t}|f_{\bM_t}(\pi) - f_{M^*}(\pi)| \leq \sqrt{T \sum_{t = 1}^T  \mathbb{E}_{p_t}(f_{\bM_t}(\pi) - f_{M^*}(\pi))^2} \leq \sqrt{T \mathsf{Est}(T, \delta)}. 
\end{align}
 
Putting these three bounds together, and using the fact that $f_{\bM_t}(\pi_{\bM_t}) \leq \sup_{M \in \cM} f_M(\pi_M) \leq 1$, we achieve
\begin{align}
    \mathsf{Reg}_{\gamma}(T) \leq (1 - \gamma)T + 2\sum_{t = 1}^T \mathsf{dec}^{\gamma}_{\eps_{\lceil{t/2}\rceil}}  + 9\sqrt{T\mathsf{Est}(T, \delta)}. 
\end{align}

Under Assumption~\ref{assm:growth}, using Lemma~\ref{lemma:holder}, we have 
\begin{align}
    \sum_{t = 1}^T \mathsf{dec}^{\gamma}_{\eps_{\lceil{t/2}\rceil}} \leq 2\sum_{t = 1}^{T/2} \mathsf{dec}^{\gamma}_{\eps_{t}} \leq 2\sum_{t = 1}^{T/2} \Delta_{\eps_{t}} \leq 2T\Delta_{\eps_{T/2}}.
\end{align}
Thus we have
\begin{align}
    \mathsf{Reg}_{\gamma}(T) \leq (1 - \gamma)T + 2T\Delta_{\sqrt{\frac{24\mathsf{Est}(T, \delta)}{T}}}  + 9\sqrt{T\mathsf{Est}(T, \delta)}. 
\end{align}
\end{proof}

\subsection{Proof of Theorem~\ref{thm:ub}}
We restate and prove Theorem~\ref{thm:ub}.
\bestub*
\begin{algorithm}[t]
\caption{}\label{alg:best}
\textbf{input:}  Parameter $\gamma \in [0, 1]$.
\begin{algorithmic}[0]
    \State Initialize epoch index $i = 1$. Let $T_0 := 0$.
    \State Initialize an online estimation oracle for the model class $\cM$.
    \State Initialize $m_1 = \sup_{M \in \cM}f_M(\pi_M)$.
    \For{$t=1,\ldots,T$}
    \Statex \algcommentbig{Obtain estimator $\bM_t$ from the Estimation oracle.}
        \State Let $\bM_t = \mathsf{AlgEst}(\mathcal{H}_{t - 1})$.
    \Statex \algcommentbig{Choose $\eps_t$.}
    \State Approximate the maximum as $v_t = m_i - \frac{10}{\gamma}\beta_{t - T_{i-1}}$, where $\beta_s := \Delta_{\sqrt{\frac{\mathsf{Est}(T, \delta) + 1}{s}}}$. 
    \State Let $\eps_t := \min_{\eps}: g(\eps, \bM_t, v_t) \leq 0$ where $g$ is defined in Definition~\ref{def:G}.
    \Statex \algcommentbig{Play an action from the distribution that minimizes $\mathsf{dec}^{\gamma}_{\eps_t}$.}
   \State Play $\pi_t \sim p_t$ where $p_t := p(\eps_t, \bM_t)$ is the distribution guaranteed by Assumption~\ref{assm:cont}.
   \Statex \algcommentbig{Refine the feasible model class to exclude models that cannot be $M^*$.}
    \State Let $\cM_t := \{M \in \cM: \sum_{s = 1}^t \mathbb{E}_{\pi \sim p_s}(f_M(\pi) - f_{\bM_s}(\pi))^2 \leq \mathsf{Est}(T, \delta)\}$
    \Statex, \algcommentbig{End the epoch if the feasible class doesn't contain a large enough maximum.}
       \If {$\sup_{M \in \cM_t} f_M(\pi_M) < m_i - \frac{10}{\gamma}\beta_{t + 1 - T_{i - 1}}$}
       \State Start a new epoch: $i \leftarrow i + 1$
       \State Record the epoch end time as $T_i := t$, and the epoch length as $L_i = T_i - T_{i - 1}$.
       \State Let $m_i = \sup_{M \in \cM_t}f_M(\pi_M)$.
       \EndIf
    \EndFor
\end{algorithmic}
\end{algorithm}

We use the following lemma in the proofs of Proposition~\ref{thm:ubprop} and Theorem~\ref{thm:ub}.

\begin{lemma}\label{lemma:holder}
If $\Delta_{\eps} := C\eps^{c}$ for some $0 \leq c \leq 1$, then for any integer $T$ and any $Q > 0$, we have
\begin{align}
    \sup_{x \in \mathbb{R}^T, \|x\|^2 \leq Q}\sum_{t = 1}^T \Delta_{x_t} \leq T\Delta_{\sqrt{\frac{Q}{T}}}.
\end{align}
Further, we have
\begin{align}
    \sum_{t = 1}^T \Delta_{\sqrt{\frac{Q}{t}}} \leq 2T\Delta_{\sqrt{\frac{Q}{T}}}.
\end{align}
\end{lemma}
\begin{proof}
The first inequality follows from applying Holder's inequality, which states that for $X, Y \in \mathbb{R}^T$ and $\frac{1}{p} + \frac{1}{q} = 1$, we have $\frac{1}{T}\sum_t X_tY_t \leq \left(\frac{1}{T}\sum_t X_t^p\right)^{\frac{1}{p}}\left(\frac{1}{T}\sum_t Y_t^q\right)^{\frac{1}{q}}$. Apply Holder's with $X_t = \Delta_{x_t} = x_i^c$, $Y_t = 1$, and $p = \frac{2}{c}$.

For the second inequality, we can bound
\begin{align}
     \sum_{t = 1}^T \Delta_{\sqrt{\frac{Q}{t}}} &= C\sum_{t = 1}^T \left(\frac{Q}{t}\right)^{c/2}\\
     &\leq C\int_{x = 0}^T \left(\frac{Q}{x}\right)^{c/2}dx\\
     &= Q^{c/2}\frac{1}{1 - c/2}T^{1 - c/2}\\
     &\leq 2T \left(\frac{Q}{T}\right)^{c/2}\\
     &= 2T \Delta_{\sqrt{\frac{Q}{T}}}.
\end{align}
\end{proof}

\begin{proof}[Proof of Theorem~\ref{thm:ub}]
Throughout the proof, we assume we have conditioned on the event that 
\begin{align}\label{eq:est}
    \sum_{t = 1}^T \mathbb{E}_{\pi \sim p_t}(f_{M^*}(\pi) - f_{\bM_t}(\pi))^2 \leq \mathsf{Est}(T, \delta),
\end{align}
which occurs with probability $1 - \delta$.

Now we sum up the regret over the $T$ rounds. At each epoch $i$, by the first item of Lemma~\ref{lemma:G}, we have
\begin{align}
    \sum_{t = T_{i - 1} + 1}^{T_i}\mathbb{E}_{p_t}f_{\bM_t}(\pi) &\geq \sum_{t = T_{i - 1} + 1}^{T_i} \gamma v_t - 2\eps_t -  \Delta_{\eps_t} \\
    &= \sum_{t = T_{i - 1} + 1}^{T_i} \gamma \left(m_i - \frac{10}{\gamma}\beta_{T_i - T_{i -1}}\right) - 2\eps_t - \Delta_{\eps_t}\\ 
    &= \gamma L_i m_i - \sum_{t = T_{i - 1} + 1}^{T_i}  \left(10 \beta_{t - T_{i - 1}} + 2\eps_t + \Delta_{\eps_t}\right).
\end{align}

Similar to the proof of Proposition~\ref{thm:ubprop}, we next want to show that the $\eps_t$ cannot be too large. Let $M_i$ be any model such that $f_{M_i}(\pi_{M_i}) \geq v_{T_i}$, and 
\begin{align}\label{Miest}
    \sum_{t = 1}^{T_i - 1}\mathbb{E}_{\pi \sim p_t}(f_{M_i}(\pi) - f_{\bM_t}(\pi))^2 \leq \mathsf{Est}(T, \delta).
\end{align}
Such a model necessarily exists since the epoch ended at time $T_i$, meaning that at time $T_i - 1$, there was a some model satisfying the above. Thus by item 2 of Lemma~\ref{lemma:G}, at each step $t$, we have $\mathbb{E}_{p(\eps_t, \bM_t)}[(f_{\bM}(\pi) - f_{M_i}(\pi))^2] \geq \eps_t^2$. Thus by Equation~\eqref{Miest}, we have 

\begin{align}
    \sum_{t = T_{i - 1}}^{T_i} \eps_t^2 \leq \sum_{t = T_{i-1}}^{T_i}\mathbb{E}_{p(\eps_t, \bM_t)}[(f_{\bM_t}(\pi) - f_{M_i}(\pi))^2] \leq \mathsf{Est}(T, \delta) + (f_{\bM_{T_i}}(\pi) - f_{M_i}(\pi))^2 \leq \mathsf{Est}(T, \delta) + 1 
\end{align}

Thus by Assumption~\ref{assm:growth} and Lemma~\ref{lemma:holder}, with $L_i = T_i - T_{i - 1}$, we have
\begin{align}
    \sum_{t = T_{i - 1} + 1}^{T_i} \Delta_{\eps_t} \leq \sup_{x \in \mathbb{R}^{L_i}, \|x\|^2 \leq \mathsf{Est}(T, \delta) + 1}\sum_i \Delta_{x_i} \leq L_i \beta_{L_i},
\end{align}
and similarly by Lemma~\ref{lemma:holder},
\begin{align}
    \sum_{t = T_{i - 1} + 1}^{T_i}  \beta_{t - T_{i - 1}} \leq 2L_i\beta_{L_i}.
\end{align}
Thus in total, since $\eps \leq \Delta_{\eps}$ (see Assumption~\ref{assm:growth}), we have
\begin{align}\label{eq:epoch}
     \sum_{t = T_{i - 1} + 1}^{T_i}\mathbb{E}_{p_t}f_{\bM_t}(\pi) &\geq \gamma L_i m_i - (10*2 + 3)L_i\beta_{L_i} = \gamma L_i m_i - 23L_i\beta_{L_i}.
\end{align}

Now consider the total $\gamma$-regret we accumulate after all $T$ rounds. Let $I$ be the total number of epochs. Then we have 
\begin{align}\label{eq:original}
    \gamma T f_{M^*}(\pi_{M^*}) - \sum_{t = 1}^{T}\mathbb{E}_{p_t}f_{M^*}(\pi) &\leq \gamma T f_{M^*}(\pi_{M^*}) - \sum_{t = 1}^{T}\mathbb{E}_{p_t}f_{\bM_t}(\pi) + \sum_{t = 1}^T\mathbb{E}_{p_t}|f_{\bM_t}(\pi) - f_{M^*}(\pi)|\\
    &\leq \gamma T f_{M^*}(\pi_{M^*}) - \sum_{t = 1}^{T}\mathbb{E}_{p_t}f_{\bM_t}(\pi) + \sqrt{T\mathsf{Est}(T, \delta)}\\
    &\leq \gamma \sum_{i = 1}^I L_i (f_{M^*}(\pi_{M^*}) - m_i) + 23 L_i\beta_{L_i} +  \sqrt{T\mathsf{Est}(T, \delta)}.
\end{align}
Here in the second line we applied Equation~\eqref{eq:est} and Jenson's inequality to the distribution $\frac{1}{T}\sum_t p_t$, and in the third line we plugged in Equation~\eqref{eq:epoch}.

Observe crucially that for all $i \leq I - 1$, we must have $m_{i + 1} \leq m_i - \frac{10}{\gamma}\beta_{L_i + 1}$ because we ended the epoch precisely because there was no model left in $\tilde{\cM}_t$ with small estimation error with a maximum above $m_i - \frac{10}{\gamma}\beta_{L_i + 1}$. Using Assumption~\ref{assm:growth}, we can simplify to get $m_{i + 1} \leq m_i - \frac{9}{\gamma}\beta_{L_i}$. Indeed, the epoch must last at least 99 rounds since $\frac{10}{\gamma}\beta_{100} \geq 1$, $\beta_{L_i + 1} \geq 0.99 \beta_{L_i}$ for $L_i \geq 99$.

Further, we must have $m_{I} \geq f_{M^*}(\pi_{M^*})$ because we have conditioned on the event that $\sum_{t = 1}^T \mathbb{E}_{\pi \sim p_t}(f_{M^*}(\pi) - f_{\bM_t}(\pi))^2 \leq \mathsf{Est}(T, \delta)$.

If $I$ is a constant, it is easy to check that this expression is sufficiently large to guarantee a $\gamma$-regret of $O(\beta_T)$. For instance, if $I = 1$, then this follows immediately. In the rest of the proof, our goal will be to show that we can essentially reduce to the case when $I \leq O(\log(T))$.

Consider maximizing the expression 
\begin{align}\label{eq:expr}
    \gamma \sum_{i = 1}^I L_i (f_{M^*}(\pi_{M^*}) - m_i) + 23 L_i\beta_{L_i}
\end{align}
over any choice of feasible integer $I$, integers $\{L_i\}_{i \in [I]}$ and $\{m_i\}_{i \in I} \in [0, 1]$ such that $\sum_{i = 1}^I L_i = T$, and $m_{I} \geq f_{M^*}(\pi_{M^*})$ and $m_{i + 1} \leq m_i -\frac{9}{\gamma}\beta_{L_i}$.

We make the follow two claims about the optimal solution.

\begin{claim}
    In any optimal solution, we will have $m_{i + 1} = m_i -\frac{9}{\gamma}\beta_{L_i}$ exactly, and $m_I = f_{M^*}(\pi_{M^*})$.  Thus in any optimal solution, we have $m_i - f_{M^*}(\pi_{M^*}) = \sum_{j = i}^{I - 1} \frac{9}{\gamma}\beta_{L_j}$.
\end{claim}
\begin{proof}
    Working backwards from $i = I$ to $1$, we can see that decreasing the $m_{i}$ as much as possible clearly increases the objective.
\end{proof}

Thus we now care about upper bounding the objective
\begin{align}\label{eq:obj}
    -\sum_i L_i\sum_{j = i}^{I - 1} 9\beta_{L_j} + 23 L_i\beta_{L_i} &\leq - \sum_{i=1}^{I - 1} \left(9L_i \sum_{j = i + 1}^{I - 1} \beta_{L_j} + 14L_i\beta_{L_i}\right) + 23L_I\beta_{L_I}\\
    &\leq - \sum_{i=1}^{I - 1} \left(9L_i \sum_{j = i + 1}^{I - 1} \beta_{L_j} + 14L_i\beta_{L_i}\right) + 23T\beta_T,
\end{align}
where the last inequality follows from Assumption~\ref{assm:growth}.

\begin{claim}\label{claim:combine}
    If for some $k \in [I - 2]$, we have $L_{k + 1} \leq 1.1L_k$, then we can increase the value of the above upper bound on the objective by combining epochs $k$ and $k + 1$ into a single epoch of length $L_k + L_{k + 1}$.
\end{claim}
\begin{proof}
For convenience, let $S_i := \sum_{j = i + 1}^{I - 1} \beta_{L_j}$.

We have two terms that depend on the $L_i$ in the above objective, $-9\sum_i L_i S_i$ and $14 \sum_i L_i\beta_{L_i}$.

Consider doing this change where we combine $L_k$ and $L_{k + 1}$.

For the first term involving the $S_i$, for all the terms with $i \leq k - 1$, the value $L_iS_i$ will only decrease from combining, since trivially, $\beta_{L_k} + \beta_{L_{k + 1}} \geq \beta_{L_k + L_{k + 1}}$. For all the terms with $i \geq k + 1$, the value of $L_iS_i$ will be unaffected.

In combining $L_k$ and $L_{k + 1}$, we will lose the term $L_k S_k +L_{k + 1}S_{k+1}$ and gain $(L_k + L_{k + 1})S_{k + 1}$.

Thus by combining the two epochs, the total objective increases by at least 
\begin{align}
   9 \left(L_k S_k +L_{k + 1}S_{k+1}\right) - 9(L_k + L_{k + 1})S_{k + 1} - 14 L_k \beta_{L_k} - 14 L_{k + 1}\beta_{L_{k + 1}} + 14(L_k + L_{k + 1}) \beta_{L_k + L_{k+1}}.
\end{align}
Here the final 3 terms arise from the change to the $14 \sum_i L_i\beta_{L_i}$ term in the objective. Simplifying, this increase equals

\begin{align}
    9L_k&(S_k - S_{k + 1}) - 14 L_k \beta_{L_k} - 14L_{k + 1}\beta_{L_{k + 1}} + 14(L_k + L_{k + 1}) \beta_{L_k + L_{k+1}}\\
    &= 9L_k\beta_{L_{k+1}} - 14 L_k \beta_{L_k} - 14L_{k + 1}\beta_{L_{k + 1}} + 14(L_k + L_{k + 1}) \beta_{L_k + L_{k+1}}\\
    &= 9\left(L_k\beta_{L_{k+1}} - \alpha L_k \beta_{L_k} - \alpha L_{k + 1}\beta_{L_{k + 1}} + \alpha (L_k + L_{k + 1}) \beta_{L_k + L_{k+1}}\right),
\end{align}
where $\alpha = 14/9$. 

\begin{claim}
    if $L_{k + 1}/L_k \leq 1.1$, then $$9\left(L_k\beta_{L_{k+1}} - \alpha L_k \beta_{L_k} - \alpha L_{k + 1}\beta_{L_{k + 1}} + \alpha (L_k + L_{k + 1}) \beta_{L_k + L_{k+1}}\right) \geq 0.$$ 
\end{claim}
\begin{proof}
We will analyze this by plugging Assumption~\ref{assm:growth}, which states that $\Delta_{\eps} = C\eps^c$ for some $c \leq 1$. Recall that $\beta_{\ell} = \Delta_{\sqrt{\frac{\mathsf{Est}(T, \delta)}{\ell}}} = C\ell^{-c/2}(\mathsf{Est}(T, \delta))^{c/2}$.

To simplify, let $L = L_k$, and let $x = {L_{k + 1}}/L_k$. So we can lower bound the increase as 
\begin{align}
B &:= 9L\left(\beta_{Lx} - \alpha  \beta_{L} - \alpha x\beta_{Lx} + \alpha (1 + x) \beta_{L(1 + x)}\right)\\
    &= 9LC\left(\frac{\mathsf{Est}(T, \delta)}{L}\right)^{c/2}\left(x^{-c/2} - \alpha  - \alpha x x^{-c/2} + \alpha (1 + x) (1 + x)^{-c/2}\right)\\
    &= 9LC\left(\frac{\mathsf{Est}(T, \delta)}{L}\right)^{c/2}\left(x^{-c/2} - \alpha  - \alpha x^{1-c/2} + \alpha (1 + x)^{1-c/2}\right).
\end{align}

Applying Holder's inequality as in Lemma~\ref{lemma:holder}, since $c \leq 1$ we have 
\begin{align}
    1 + x^{1-c/2} = 1^{1 -c/2} + x^{1-c/2} \leq (1 + x)\left(\frac{1 + x}{2}\right)^{-c/2},
\end{align}
and thus, we have
\begin{align}
    B &\geq 9LC\left(\frac{\mathsf{Est}(T, \delta)}{L}\right)^{c/2}\left(x^{-c/2} - \alpha (1 + x)\left(\frac{1 + x}{2}\right)^{1 - c/2} + \alpha (1 + x)^{1-c/2}\right)\\
    &= 9LC\left(\frac{\mathsf{Est}(T, \delta)}{L}\right)^{c/2}\left(x^{-c/2} + \alpha \left(-2^{c/2} + 1\right)(1 + x)^{1-c/2})\right)\\
    &= 9LC\left(\frac{\mathsf{Est}(T, \delta)}{L}\right)^{c/2}x^{-c/2}\left(1 + \alpha \left(-2^{c/2} + 1\right)(1 + x)\left(\frac{x}{1 + x}\right)^{c/2}\right)
\end{align}
Now we want to show that this is positive for $x \leq 1.1$. Observe that for $c \geq 0$, the expression  
\begin{align}
    1 + \alpha \left(-2^{c/2} + 1\right)(1 + x)\left(\frac{x}{1 + x}\right)^{c/2}
\end{align} is non-increasing in $x$ for $x \geq 0$. Thus it suffices to show that when we plug in $x = 1.1$, that expression is positive. Plugging in $x = 1.1$, this evaluates to 
\begin{align}
    1 + \alpha \left(-2^{c/2} + 1\right)2.1\left(\frac{1.1}{2.1}\right)^{c/2} \geq 1 + 3.3\left(-2^{c/2} + 1\right)\left(\frac{1.1}{2.1}\right)^{c/2}
\end{align}
One can check that the above expression is decreasing in $c$, and thus since $c \leq 1$, the expression is at least $1 + 3.3\left(-2^{1/2} + 1\right)\left(\frac{1.1}{2.1}\right)^{1/2} > 0$.
\end{proof}
It follows from this claim that if $x = L_{k + 1}/L_k < 1.1$, it will increase the objective to combine the two epochs. This yields the claim.

\end{proof}

It follows from Claim~\ref{claim:combine} that the maximum of the objective in \ref{eq:obj} is achieved when $I \leq \log_{1.1}(T) \leq 10.5\log(T)$. 

Now recalling that $\beta_s = \Delta_{\sqrt{\frac{\mathsf{Est}(T, \delta) + 1}{s}}}$, we have 

\begin{align}
    \sum_{i = 1}^{I-1} 14 L_i\beta_{L_i} + 23T\beta_T &= \sum_{i = 1}^{I - 1}\sum_{t = T_{i -1}}^{T_i} \Delta_{\sqrt{\frac{\mathsf{Est}(T, \delta) + 1}{L_i}}} + 23T\beta_T\\
    &\leq \sup_{x \in \mathbb{R}^T, \|x\|^2 \leq I (\mathsf{Est}(T, \delta) + 1)}\sum_{t = 1}^T 14 \Delta_{x_t}  + 23T\beta_T\\
    &\leq 14T\Delta_{\sqrt{\frac{I\mathsf{Est}(T, \delta) + 1}{T}}} + 23T\beta_T\\
    &\leq 37T\Delta_{\sqrt{\frac{11\log(T)\mathsf{Est}(T, \delta)}{T}}}\\
    &\leq 500\log(T)\Delta_{\sqrt{\frac{\mathsf{Est}(T, \delta)}{T}}}.
\end{align}
The second inequality and the final inequality follows from Assumption~\ref{assm:growth}. Thus we have that the objective in Equation~\eqref{eq:obj} and thus also the objective in Equation~\eqref{eq:expr} is upper bounded by $500\log(T)\Delta_{\sqrt{\frac{\mathsf{Est}(T, \delta)}{T}}}$. 

Combining with Equation~\eqref{eq:original} yields the theorem.

\end{proof}

\section*{Acknowledgements} We acknowledge the support from the ARO through award W911NF-21-1-0328 and from the DOE through award DE-SC0022199. Thanks to Kefan Dong for useful feedback and conversations.

\bibliographystyle{plainnat}
\bibliography{arxiv2}

\begin{thebibliography}{27}
\providecommand{\natexlab}[1]{#1}
\providecommand{\url}[1]{\texttt{#1}}
\expandafter\ifx\csname urlstyle\endcsname\relax
  \providecommand{\doi}[1]{doi: #1}\else
  \providecommand{\doi}{doi: \begingroup \urlstyle{rm}\Url}\fi

\bibitem[Azar et~al.(2022)Azar, Fiat, and Fusco]{azar2022alpha}
Yossi Azar, Amos Fiat, and Federico Fusco.
\newblock An alpha-regret analysis of adversarial bilateral trade.
\newblock \emph{Advances in Neural Information Processing Systems},
  35:\penalty0 1685--1697, 2022.

\bibitem[Blum et~al.(2008)Blum, Hajiaghayi, Ligett, and Roth]{blum2008regret}
Avrim Blum, MohammadTaghi Hajiaghayi, Katrina Ligett, and Aaron Roth.
\newblock Regret minimization and the price of total anarchy.
\newblock In \emph{Proceedings of the fortieth annual ACM symposium on Theory
  of computing}, pages 373--382, 2008.

\bibitem[Bubeck et~al.(2012)Bubeck, Cesa-Bianchi, et~al.]{bubeck2012regret}
S{\'e}bastien Bubeck, Nicolo Cesa-Bianchi, et~al.
\newblock Regret analysis of stochastic and nonstochastic multi-armed bandit
  problems.
\newblock \emph{Foundations and Trends{\textregistered} in Machine Learning},
  5\penalty0 (1):\penalty0 1--122, 2012.

\bibitem[Chen et~al.(2016)Chen, Hu, Li, Li, Liu, and Lu]{chen2016combinatorial}
Wei Chen, Wei Hu, Fu~Li, Jian Li, Yu~Liu, and Pinyan Lu.
\newblock Combinatorial multi-armed bandit with general reward functions.
\newblock \emph{Advances in Neural Information Processing Systems}, 29, 2016.

\bibitem[Dud{\'\i}k et~al.(2020)Dud{\'\i}k, Haghtalab, Luo, Schapire,
  Syrgkanis, and Vaughan]{dudik2020oracle}
Miroslav Dud{\'\i}k, Nika Haghtalab, Haipeng Luo, Robert~E Schapire, Vasilis
  Syrgkanis, and Jennifer~Wortman Vaughan.
\newblock Oracle-efficient online learning and auction design.
\newblock \emph{Journal of the ACM (JACM)}, 67\penalty0 (5):\penalty0 1--57,
  2020.

\bibitem[Foster and Rakhlin(2022)]{FosRak22}
D.~Foster and A.~Rakhlin.
\newblock
  \url{https://www.mit.edu/~rakhlin/courses/course_stat_rl/course_stat_rl.pdf},
  2022.
\newblock Course: Statistical Reinforcement Learning and Decision Making,
  Fall'22, MIT.

\bibitem[Foster and Rakhlin(2021)]{foster2021submodular}
Dean~P Foster and Alexander Rakhlin.
\newblock On submodular contextual bandits.
\newblock \emph{arXiv preprint arXiv:2112.02165}, 2021.

\bibitem[Foster and Rakhlin(2020)]{foster2020beyond}
Dylan Foster and Alexander Rakhlin.
\newblock Beyond ucb: Optimal and efficient contextual bandits with regression
  oracles.
\newblock In \emph{International Conference on Machine Learning}, pages
  3199--3210. PMLR, 2020.

\bibitem[Foster et~al.(2021)Foster, Kakade, Qian, and
  Rakhlin]{foster2021statistical}
Dylan~J Foster, Sham~M Kakade, Jian Qian, and Alexander Rakhlin.
\newblock The statistical complexity of interactive decision making.
\newblock \emph{arXiv preprint arXiv:2112.13487}, 2021.

\bibitem[Foster et~al.(2022)Foster, Rakhlin, Sekhari, and
  Sridharan]{foster2022complexity}
Dylan~J Foster, Alexander Rakhlin, Ayush Sekhari, and Karthik Sridharan.
\newblock On the complexity of adversarial decision making.
\newblock \emph{Advances in Neural Information Processing Systems},
  35:\penalty0 35404--35417, 2022.

\bibitem[Foster et~al.(2023)Foster, Golowich, and Han]{foster2023tight}
Dylan~J Foster, Noah Golowich, and Yanjun Han.
\newblock Tight guarantees for interactive decision making with the
  decision-estimation coefficient.
\newblock \emph{arXiv preprint arXiv:2301.08215}, 2023.

\bibitem[Fotakis et~al.(2020)Fotakis, Lianeas, Piliouras, and
  Skoulakis]{fotakis2020efficient}
Dimitris Fotakis, Thanasis Lianeas, Georgios Piliouras, and Stratis Skoulakis.
\newblock Efficient online learning of optimal rankings: Dimensionality
  reduction via gradient descent.
\newblock \emph{Advances in Neural Information Processing Systems},
  33:\penalty0 7816--7827, 2020.

\bibitem[Garber(2017)]{garber2017efficient}
Dan Garber.
\newblock Efficient online linear optimization with approximation algorithms.
\newblock \emph{Advances in Neural Information Processing Systems}, 30, 2017.

\bibitem[Harvey et~al.(2020)Harvey, Liaw, and Soma]{harvey2020improved}
Nicholas Harvey, Christopher Liaw, and Tasuku Soma.
\newblock Improved algorithms for online submodular maximization via
  first-order regret bounds.
\newblock \emph{Advances in Neural Information Processing Systems},
  33:\penalty0 123--133, 2020.

\bibitem[Hazan et~al.(2018)Hazan, Hu, Li, and Li]{hazan2018online}
Elad Hazan, Wei Hu, Yuanzhi Li, and Zhiyuan Li.
\newblock Online improper learning with an approximation oracle.
\newblock \emph{Advances in Neural Information Processing Systems}, 31, 2018.

\bibitem[Ito et~al.(2019)Ito, Hatano, Sumita, Takemura, Fukunaga, Kakimura, and
  Kawarabayashi]{ito2019oracle}
Shinji Ito, Daisuke Hatano, Hanna Sumita, Kei Takemura, Takuro Fukunaga,
  Naonori Kakimura, and Ken-Ichi Kawarabayashi.
\newblock Oracle-efficient algorithms for online linear optimization with
  bandit feedback.
\newblock \emph{Advances in Neural Information Processing Systems}, 32, 2019.

\bibitem[Kakade et~al.(2007)Kakade, Kalai, and Ligett]{kakade2007playing}
Sham~M Kakade, Adam~Tauman Kalai, and Katrina Ligett.
\newblock Playing games with approximation algorithms.
\newblock In \emph{Proceedings of the thirty-ninth annual ACM symposium on
  Theory of computing}, pages 546--555, 2007.

\bibitem[Lattimore and Szepesv{\'a}ri(2020)]{lattimore2020bandit}
Tor Lattimore and Csaba Szepesv{\'a}ri.
\newblock \emph{Bandit algorithms}.
\newblock Cambridge University Press, 2020.

\bibitem[Niazadeh et~al.(2021)Niazadeh, Golrezaei, Wang, Susan, and
  Badanidiyuru]{niazadeh2021online}
Rad Niazadeh, Negin Golrezaei, Joshua~R Wang, Fransisca Susan, and Ashwinkumar
  Badanidiyuru.
\newblock Online learning via offline greedy algorithms: Applications in market
  design and optimization.
\newblock In \emph{Proceedings of the 22nd ACM Conference on Economics and
  Computation}, pages 737--738, 2021.

\bibitem[Nie et~al.(2022)Nie, Agarwal, Umrawal, Aggarwal, and
  Quinn]{nie2022explore}
Guanyu Nie, Mridul Agarwal, Abhishek~Kumar Umrawal, Vaneet Aggarwal, and
  Christopher~John Quinn.
\newblock An explore-then-commit algorithm for submodular maximization under
  full-bandit feedback.
\newblock In \emph{Uncertainty in Artificial Intelligence}, pages 1541--1551.
  PMLR, 2022.

\bibitem[Paria and Sinha(2021)]{paria2021texttt}
Debjit Paria and Abhishek Sinha.
\newblock Leadcache: Regret-optimal caching in networks.
\newblock \emph{Advances in Neural Information Processing Systems},
  34:\penalty0 4435--4447, 2021.

\bibitem[Perrault(2022)]{perrault2022combinatorial}
Pierre Perrault.
\newblock When combinatorial thompson sampling meets approximation regret.
\newblock \emph{Advances in Neural Information Processing Systems},
  35:\penalty0 17639--17651, 2022.

\bibitem[Rajaraman et~al.(2023)Rajaraman, Han, Jiao, and
  Ramchandran]{rajaraman2023beyond}
Nived Rajaraman, Yanjun Han, Jiantao Jiao, and Kannan Ramchandran.
\newblock Beyond ucb: Statistical complexity and optimal algorithms for
  non-linear ridge bandits.
\newblock \emph{arXiv preprint arXiv:2302.06025}, 2023.

\bibitem[Roughgarden and Wang(2019)]{roughgarden2019minimizing}
Tim Roughgarden and Joshua~R Wang.
\newblock Minimizing regret with multiple reserves.
\newblock \emph{ACM Transactions on Economics and Computation (TEAC)},
  7\penalty0 (3):\penalty0 1--18, 2019.

\bibitem[Streeter and Golovin(2008)]{streeter2008online}
Matthew Streeter and Daniel Golovin.
\newblock An online algorithm for maximizing submodular functions.
\newblock \emph{Advances in Neural Information Processing Systems}, 21, 2008.

\bibitem[Yang et~al.(2021)Yang, Chen, Zhang, and Sun]{yang2021follow}
Feidiao Yang, Wei Chen, Jialin Zhang, and Xiaoming Sun.
\newblock Follow the perturbed approximate leader for solving semi-bandit
  combinatorial optimization.
\newblock \emph{Frontiers of Computer Science}, 15:\penalty0 1--12, 2021.

\bibitem[Zhang et~al.(2019)Zhang, Chen, Hassani, and Karbasi]{zhang2019online}
Mingrui Zhang, Lin Chen, Hamed Hassani, and Amin Karbasi.
\newblock Online continuous submodular maximization: From full-information to
  bandit feedback.
\newblock \emph{Advances in Neural Information Processing Systems}, 32, 2019.

\end{thebibliography}

\end{document}